\documentclass[sigconf]{acmart} %% review,anonymous
\usepackage{subcaption}
\usepackage{multirow}
\usepackage{amsmath}
\usepackage{amsthm}

\definecolor{deepred}{rgb}{0.66, 0.0, 0.0} % 自定义深红色

\theoremstyle{plain}
\newtheorem{theorem}{Theorem}
\newtheorem{lemma}{Lemma}
\newtheorem{corollary}{Corollary}

\theoremstyle{definition}

\theoremstyle{remark}

\AtBeginDocument{%
  \providecommand\BibTeX{{%
    \normalfont B\kern-0.5em{\scshape i\kern-0.25em b}\kern-0.8em\TeX}}}

\setcopyright{acmlicensed}
\copyrightyear{2018}
\acmYear{2018}
\acmDOI{XXXXXXX.XXXXXXX}

\acmConference[Conference acronym 'XX]{Make sure to enter the correct
  conference title from your rights confirmation emai}{June 03--05,
  2018}{Woodstock, NY}
\acmISBN{978-1-4503-XXXX-X/18/06}

\acmSubmissionID{168}

\begin{document}

\title{Enhancing Monotonic Modeling with Spatio-Temporal Adaptive Awareness in Diverse Marketing}


\author{Bin Li, HengXu He, Jiayan Pei, Feiyang Xiao, Yifan Zhao, Zhixing Zhang, Diwei Liu, Jia Jia}
\affiliation{%
  \institution{Alibaba Group}
  \streetaddress{P.O. Box 1212}
  \city{Hangzhou}
  \state{Zhejiang}
  \country{China}
  \postcode{43017-6221}}
\email{{ningzhu.lb,hengxu.hhx,jiayanpei.pjy,feiyangxiao.xfy,zhang.zzx,jj229618}@alibaba-inc.com,{xuefan.zyf,diwei.ldw}@koubei.com}

\renewcommand{\shortauthors}{Bin Li and HengXu He, et al.}

\begin{abstract}
In the mobile internet era, the Online Food Ordering Service (OFOS) emerges as an integral component of inclusive finance owing to the convenience it brings to people. OFOS platforms offer dynamic allocation incentives to users and merchants through diverse marketing campaigns to encourage payments while maintaining the platforms' budget efficiency.
Despite significant progress, the marketing domain continues to face two primary challenges: (i) how to allocate a limited budget with greater efficiency, demanding precision in predicting users' monotonic response (i.e. sensitivity) to incentives, and (ii) ensuring spatio-temporal adaptability and robustness in diverse marketing campaigns across different times and locations.
To address these issues, we propose a \textit{Constrained Monotonic Adaptive Network} (CoMAN) method for spatio-temporal perception within marketing pricing. Specifically, we capture spatio-temporal preferences within attribute features through two foundational spatio-temporal perception modules. To further enhance catching the user sensitivity differentials to incentives across varied times and locations, we design modules for learning spatio-temporal convexity and concavity as well as for expressing sensitivity functions.
CoMAN can achieve a more efficient allocation of incentive investments during pricing, thus increasing the conversion rate and orders while maintaining budget efficiency. Extensive offline and online experimental results within our diverse marketing campaigns demonstrate the effectiveness of the proposed approach while outperforming the monotonic \textit{state-of-the-art} method.
\end{abstract}


% \ccsdesc[500]{Information systems~Spatial-temporal systems; Online advertising}




%received

\maketitle

\section{Introduction}
%% lbs takeaway&marketing background
With mobile payment services, Online Food Ordering Service (OFOS) platforms such as Ele.me,\footnote{\href{https://www.ele.me.com}{www.ele.me.com}} Meituan \footnote{\href{https://www.meituan.com}{www.meituan.com}} and Uber Eats\footnote{\href{https://www.ubereats.com}{www.ubereats.com}} have become critical components of inclusive finance through the convenience offered by real-time provisioning.
To encourage mobile payments for food and retail delivery, platforms offer incentives (e.g., coupons, commissions, bonuses) to users through diverse marketing strategies and campaigns. Within these marketing campaigns, the intelligent allocation of incentives to users and merchants represents a critical aspect for platforms in enhancing marketing efficiency.

% \begin{figure}[h]
% \centering
% \includegraphics[width=8.5cm]{figs/background.pdf}
% \caption{Monotonic Properties of Marketing Campaign. Enclosed within the red box in the figure are the monotonic characteristics of various marketing campaigns. Sequentially from left to right, the attributes represented are the discount thresholds and amounts in the tiered discount, followed by the actual and nominal delivery fees in the delivery fee waiver, and concluding with the coupon amounts and thresholds in the exploding red packets.}
% \label{fig:background}
% \end{figure}

%% monotonic response
Users exhibit different monotone response patterns to platform incentives across various marketing campaigns\cite{yu2021joint,zhao2019unified}. It is reasonable that a higher incentive amount increases the probability of user redemption. Specifically, the user Conversion Rate (CVR) expresses a monotonic increasing trend to the increment of red packets or discount amounts, while it monotonically decreases about the threshold for coupon. Similarly, it exhibits a monotonic decreasing relationship with delivery fees. Without incorporating the monotone prior into the predictive model of user response, the model is susceptible to learning biased and over-fit patterns influenced by the noise introduced in the dataset.

%% marketing paradigm & s-t background
Previous work in the marketing domain \cite{boutilier2016budget,beheshti2015survey} has explored various methods for optimizing incentives under a limited budget. \cite{ito2017optimization,zhao2019unified,yu2021joint} have approached marketing pricing by employing a paradigm consisting of two components: predicting user response scores and making real-time decisions through linear programming. 
Importantly, in the takeaway domain, where user preferences are significantly influenced by spatio-temporal factors, relevant studies in the recommendation field \cite{lin2022spatiotemporal,du2023basm,lin2023exploring} have explored the learning and enhanced utilization of spatio-temporal information.

%% motivation
%% i) precision ii) adaptability in spatio-temporal/diverse
However, previous research in the takeaway marketing domain has overlooked the critical importance of perceiving spatio-temporal information, focusing instead on analyzing users' sensitivity to incentives in a general context. In reality, users exhibit varying sensitivities to platform incentives when ordering from different merchants at different times and locations. Relying solely on a uniform sensitivity function to assess user response across varying times and locations is flawed. Therefore, it is essential to differentiate descriptions of user sensitivity by spatio-temporal information awareness. Meanwhile, ensuring the model possesses enhanced predictive accuracy and adapts robustly to diverse marketing.

% \begin{figure}[h]
% \centering
% \includegraphics[width=8.5cm]{figs/time-aware-marketing.pdf}
% \caption{Illustration of differences between traditional takeaway marketing and time-aware takeaway marketing. The curves in the figure represent sensitivity curves that detail how user conversion rates change with increasing incentives.}
% \label{fig:time-aware-marketing}
% \end{figure}

In this paper, we tackle these challenges by proposing \textit{Constrained Monotonic Adaptive Network} (CoMAN), a novel marketing pricing model utilizing spatiotemporal-enhanced monotonic response. Specifically, our method employs two spatio-temporal aware activation modules to guide the model in learning and understanding spatio-temporal features, thus facilitating a more reasonable capture of spatio-temporal information within the representations. Additionally, our meticulously designed monotonic layer enhances the perception of the spatio-temporal differentiation in users’ monotonic responses to incentives. Concurrently, we introduce the activation function tailored to adaptively learn the incentive sensitivity function of users in the marketing domain. Ultimately, CoMAN achieves a strict continuous monotonic, precise, and smooth fitting of the incentive sensitivity function across different spatio-temporal dimensions. The effectiveness of our proposed methodology is demonstrated through extensive online and offline experiments across diverse marketing campaigns. In summary, our main contributions are as follows:

\begin{itemize}
\item We highlight the primary challenges encountered in the marketing domain and analyze the significance of spatio-temporal preference awareness, while further designing a novel model to address these issues. To the best of our knowledge, our model is the first to systematically model monotonic response by perceiving spatio-temporal information for diverse takeaway marketing.
\item We propose a monotonic layer that utilizes the perception of spatio-temporal information to enhance modeling monotonic responses. This module adaptively learns the concavity, convexity, and parametric expression of user incentive sensitivity functions in diverse marketing campaigns conditioned on different spatio-temporal characteristics.

\item We conceptualize a strictly monotonic \textit{Convex Linear Unit} (CLU) activation function tailored to diverse marketing. We utilize the activation function to flexibly learn the differences in monotonic function expressions. We demonstrate that this strictly convex function can more effectively approximate user incentive response functions.
\end{itemize}

% \begin{figure*}[h]
% \centering
% \includegraphics[width=14.5cm]{figs/monotonic_modeling.pdf}
% \caption{Monotonic Modeling. Road map of monotonic response modeling. Milestones in the figure: Constrain the weights\cite{archer1993application},Monotonic Networks\cite{sill1997monotonic}, Monotonic Hints\cite{sill1996monotonicity}, Monotonic Lattices\cite{milani2016fast}, DLN\cite{you2017deep}, PWL\cite{gupta2019incorporate}, COMET\cite{sivaraman2020counterexample}, MILP\cite{liu2020certified}, CMNN\cite{runje2023constrained}}
% \label{fig:monotonic_modeling}
% \end{figure*}

%%%%%%%%%%%%%%%%%%%%%%%%%%%%%%%%%%%%%%%%%%%%%%%%%%%%%%%%%%%%%%%%%%%%%%%%
\vspace{-1.5em}
\section{Related Work}
\label{Related Work}
\vspace{-0.2em}
\textbf{Activation Functions}.
Activation functions play a pivotal role in neural networks by introducing nonlinear transformations, allowing deep models to capture complex patterns beyond the reach of linear operations\cite{rumelhart1986learning,neal1992connectionist,hochreiter1991untersuchungen,nair2010rectified}.
% However, saturated activation functions such as the sigmoid\cite{rumelhart1986learning} and hyperbolic tangent (tanh)\cite{neal1992connectionist} are puzzled by the vanishing gradient problem\cite{hochreiter1991untersuchungen}. Rectified Linear Unit (ReLU)\cite{nair2010rectified} mitigate this issue due to its simplicity and efficiency, but it also suffers from the dying neuron problem. 
A series of studies like Leaky ReLU~\cite{maas2013rectifier}, PReLU~\cite{he2015delving}, and ELU~\cite{clevert2015fast} have designed ReLU variants that mitigate this issue by permitting a small, non-zero gradient to negative inputs, thereby mitigating the issue while enhancing the model performance. A series of works by \cite{shang2016understanding,zagoruyko2017diracnets,eidnes2017shifting} have investigated how to combine the output of the original ReLU with its point reflection within the same neural network layer.

\textbf{Monotonicity by architecture}.
Research methods for modeling monotonic response can be categorized into two primary types. The first involves the monotonic structure by construction, such as constrained network weights to be exclusively non-negative or non-positive, guaranteeing the model exhibits either monotonically increasing or decreasing. The other employs regularization to impose monotonicity, networks implement monotonic through the modified loss functions or the heuristic regularization terms during training. 
% The principal research progress is illustrated in the Figure~\ref{fig:monotonic_modeling}. 
%% architecture methods
Min-Max networks~\cite{sill1997monotonic,daniels2010monotone} utilize monotonic linear embedding and max-min-pooling to endow networks with broad modeling capacity but are challenging to train and practice.
%% DLN
Deep Lattice Network (DLN)\cite{you2017deep} applies constraints to pairs of lattice parameters but limits the function approximation space and necessitates a substantial number of parameters learned.
%% CMNN
CMNN \cite{runje2023constrained} can approximate any continuous monotonic function by combining non-saturated and saturated activation functions while lacking the capability to learn adaptively.
%% math model
% By the universal approximation property \cite{cybenko1989approximation,hornik1991approximation,pinkus1999approximation,daniels2010monotone}, we assume that the monotonic response model fits a mathematical monotonic function, which indicates an unspecified relationship between the objective response and monotonic variables. Moreover, the cornerstone of the method lies in the judicious construction and data-driven learning of the monotonic response function.

\begin{figure*}[t]
\centering
\includegraphics[width=15.5cm]{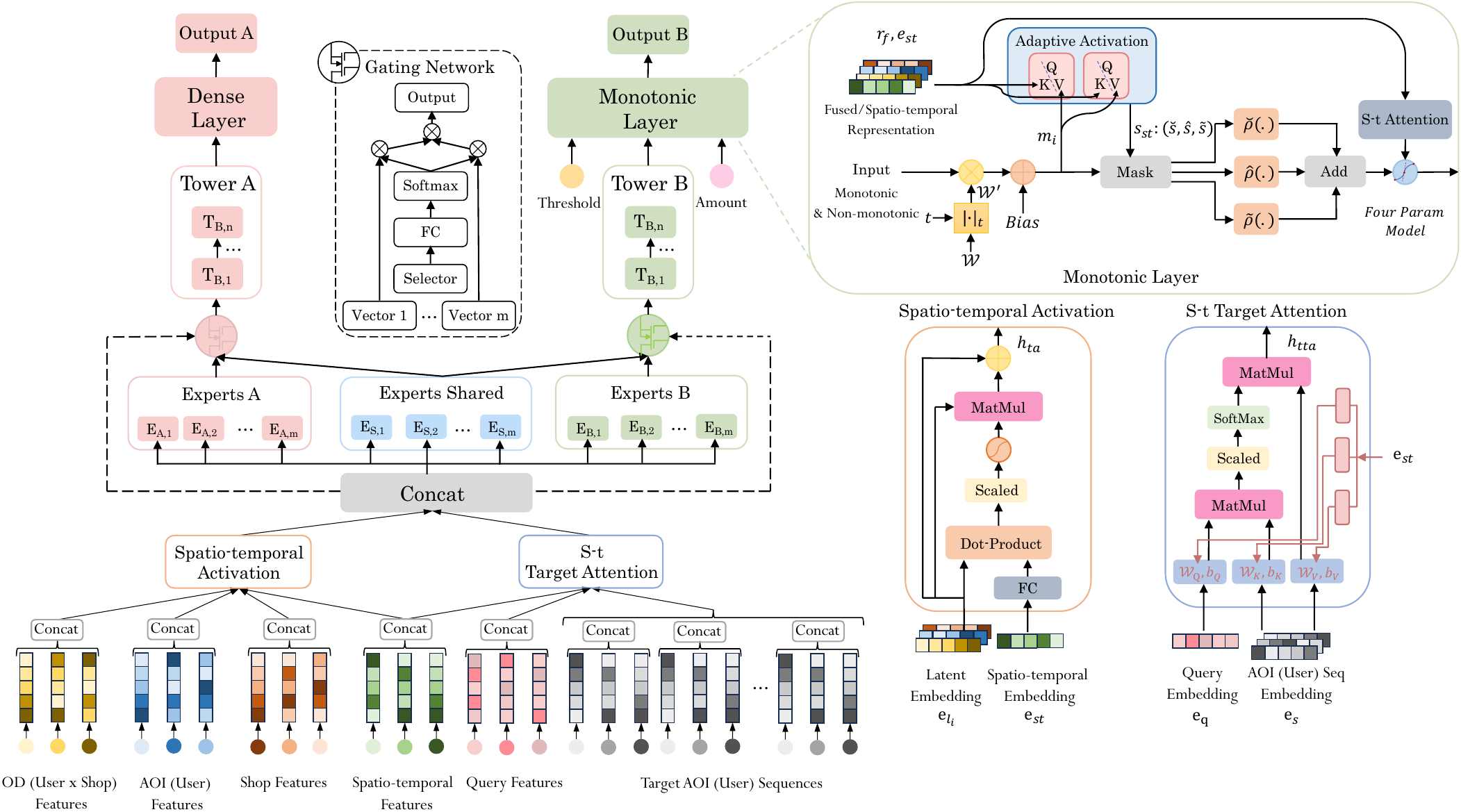}
\caption{Overview of our proposed Constrained Monotonic Adaptive Network (CoMAN) architecture.}
\vspace{-1em}
\label{fig:coman}
\end{figure*}

\textbf{Monotonicity by regularization}.
As mentioned above, enforcing monotonicity during the training can also be achieved by modifying loss function or incorporating regularization terms\cite{sill1996monotonicity,gupta2019incorporate}.
% \begin{equation}
%     E_h= \begin{cases}0 & y\left(\mathbf{x}^{\prime}\right) \geq y(\mathbf{x}) \\ \left(y(\mathbf{x})-y\left(\mathbf{x}^{\prime}\right)\right)^2 & y\left(\mathbf{x}^{\prime}\right)<y(\mathbf{x})\end{cases}.
% \end{equation}
A. Sivaraman \textit{et al.}\cite{sivaraman2020counterexample} introduced a mechanism based on Satisfiability Modulo Theories (SMT) solvers capable of identifying point pairs violating monotonicity constraints (i.e. counterexamples). Liu \textit{et al.}\cite{liu2020certified} transforms the mathematically verifying monotonicity on arbitrary piece-wise linear (i.e. ReLU) networks into a Mixed Integer Linear Programming (MILP) problem and can be efficiently solved with existing MILP solvers.
Both MILP and SMT methodologies, depend on the piecewise-linear property of ReLU to validate the constrained monotonic learning while not being extended to other activation functions such as ELU\cite{clevert2015fast}, SELU\cite{klambauer2017self}, GELU\cite{hendrycks2016gaussian}, etc. Moreover, the computational costs associated with these techniques are significant, frequently necessitating multiple iterations to achieve a rigorously verified monotonic model. 
\vspace{-0.5em}
\section{Preliminaries}
We formulate the marketing problem as a constrained linear programming problem to be solved using the optimization approach. We define the user set as \( \mathcal{C} \) and the candidate set of incentive values as \( \mathcal{T} \), where \( t_j \in \mathcal{T} \) represents the \(j\)-th treatment. Our goal is to find an incentive allocation strategy that optimally allocates incentives to maximize the sum of profits from user response (i.e. CVR) under the constrained budget \( \mathcal{B} \). Formally,
\vspace{-0.5em}
\begin{equation}
    \max_{x_{i, j}} \sum_{i=1}^{|\mathcal{C}|} \sum_{j=1}^{|\mathcal{T}|}  x_{i, j}r_{i, j} \\
\label{obj_eq}
\end{equation}
% \begin{equation}
\vspace{-1.5em}
\begin{align}
s.t.\
&x_{i, j} \in [0,1], for\ i=1,...,|\mathcal{C}|;\ j=1,...,|\mathcal{T}| \label{st1} \\
&\sum_{j=1}^{|\mathcal{T}|} x_{i, j} = 1,for\ i=1,...,|\mathcal{C}| \label{st2}\\
&\frac{\sum_{i=1}^{|\mathcal{C}|} \sum_{j=1}^{|\mathcal{T}|}  x_{i, j} t_j r_{i, j}}{\sum_{i=1}^{|\mathcal{C}|} \sum_{j=1}^{|\mathcal{T}|}  x_{i, j} r_{i, j}} \leq \mathcal{B} \label{st3}
% &\sum_{i=1}^{|\mathcal{C}|} \sum_{j=1}^{|\mathcal{T}|}  m_{i, j}^k x_{i, j} \leq h_k, for\ k =1, ..., |\mathcal{K}| \\
\end{align}
% \label{const}
% \end{equation}
% \vspace{-0.1em}
where $x_{i,j}$ denotes the decision vector indicating whether treatment $t_j$ is allocated to user $i$, and $r_{i,j}$ represents the response score by our response model. And $\sum^{|\mathcal{T}|}_{j=1}x_{i,j}t_j$ is the incentive value for the user under constraint~\ref{st2}. Furthermore, constraint~\ref{st3} indicates the limited average incentive to each user. By introducing dual variables $\mu_i$ and $\lambda$ corresponding to the aforementioned constraints \ref{st2} and \ref{st3}, we transform the primal problem into its dual counterpart. The convex optimization problem delineated in objective function \ref{obj_eq} is addressed through the application of the Lagrangian multiplier method, formally depicted as follows:
\vspace{-0.5em}
\begin{equation}
\begin{aligned}   
    \min_{x_{i, j}} &-\sum_{i=1}^{|\mathcal{C}|} \sum_{j=1}^{|\mathcal{T}|}  x_{i,j}r_{i, j}+\sum_{i=1}^{|\mathcal{C}|} (\mu_i (\sum^{|\mathcal{T}|}_{j=1}x_{i,j} - 1)) \\
    &
    +\lambda(\sum_{i=1}^{|\mathcal{C}|} \sum_{j=1}^{|\mathcal{T}|}x_{i,j}r_{i,j}(t_j -\mathcal{B}))
    \label{dual_eq}
\end{aligned}
\end{equation}
\vspace{-0.5em}
Rearrange the equation~\ref{dual_eq}, we have:
\vspace{-0.1em}
\begin{equation}
\begin{aligned}   
    \min_{x_{i, j}} &\sum_{i=1}^{|\mathcal{C}|} \sum_{j=1}^{|\mathcal{T}|}  x_{i,j}[r_{i, j}(\lambda t_j - \lambda \mathcal{B} - 1)] \\
    &+\sum_{i=1}^{|\mathcal{C}|} (\mu_i (\sum^{|\mathcal{T}|}_{j=1}x_{i,j} - 1))
\label{final_eq}
\end{aligned}
\end{equation}
\vspace{-0.1em}Given both the objective function and the constraint functions are convex, the optimal $\lambda^*$ is obtained by applying the KKT conditions and L-BFGS. From~\ref{final_eq}, it can be observed that different incentive values are only related to the first term. Consequently, an approximate solution can be derived as follows:
\vspace{-0.47em}
\begin{equation}
\begin{aligned}
    x_{i, j}= \begin{cases}1, & \text { if } j=argmin _j r_{i,j}(\lambda^* t_j - \lambda^* \mathcal{B} - 1) \\ 0, & \text { otherwise }\end{cases}
    \label{decision-eq}
\end{aligned}
\end{equation}
\vspace{-0.3em}During online pricing decisions, we employ our response model to predict response scores $|\mathcal{T}|$ times based on candidates for each request. Then utilizing the real-time computed value of $\lambda^*$ along with decision equation~\ref{decision-eq} to price the incentive value.
\vspace{-0.8em}
\section{Methods}
\label{Methods}
%% Multi-task & ple

In the marketing campaigns, we design a multi-task model as a pricing model predicting various business objectives such as Conversion Rate (CVR) or CTCVR (Click
Through \& Conversion Rate) and GMV(Gross Merchandise Volume). Given that the tasks' predictions lack substantive relevance, there is a risk of encountering the \textit{Seesaw Phenomenon} during multi-task learning. It may lead one task to improve while another deteriorates. To mitigate this, we utilize the \textit{Progressive Layered Extraction (PLE)}\cite{tang2020progressive} multi-task framework and propose the \textit{Constrained Monotonic Adaptive Network (CoMAN)}.

% Model Overall
%%%%%%%%%%%%%%%%%%%%%%%%%%%%%%%%%%%%%%%%%%%%%%%%%%%%%%%%%%%%%%%%%%%%%%%%%%%%%%%%%%%%
\textbf{Overview.} The design and modeling of our proposed framework focus on three critical aspects: activating spatio-temporal correlations within the attribute features, modeling monotonic response to incentives in diverse marketing, and enhancing the model's adaptability to spatio-temporal disparities within the monotonic layer. The overall architecture is depicted in Figure~\ref{fig:coman}.

\begin{figure*}[t]
\centering
\begin{subfigure}[b]{0.45\textwidth}
  \includegraphics[width=\linewidth]{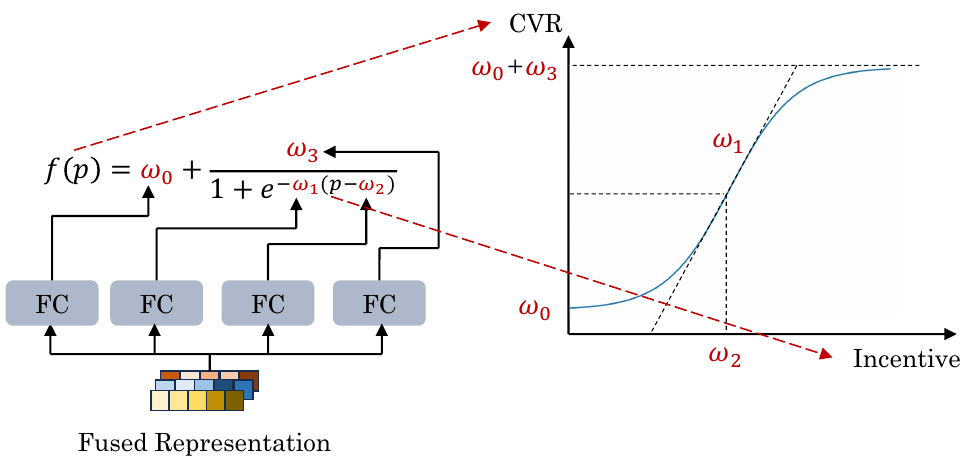}
  \caption{Four-Parameter Model}
  \label{fig:four_param_func}
\end{subfigure}\quad % \quad provides some space between the subfigures
\begin{subfigure}[b]{0.21\textwidth}
  \includegraphics[width=\linewidth]{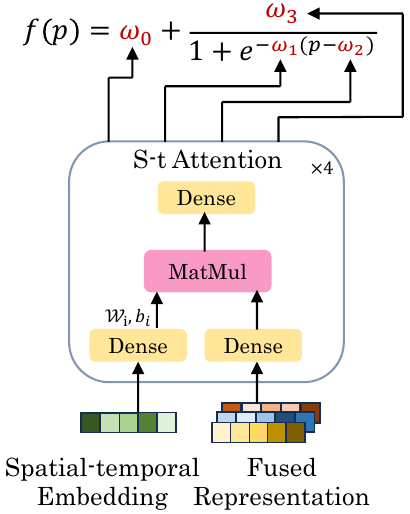}
  \caption{Spatio-temporal Attention}
  \label{fig:module_1}
\end{subfigure}
\caption{Original Four-Parameter Model (FPM) and combined with Spatio-temporal Attention (S-t Attention) fitting.}
\label{fig:main}
\end{figure*}
% Features
%%%%%%%%%%%%%%%%%%%%%%%%%%%%%%%%%%%%%%%%%%%%%%%%%%%%%%%%%%%%%%%%%%%%%%%%%%%%%%%%%%%%
In our feature extraction process, we segment the period into five intervals: breakfast [04:00–10:00), lunch [10:00–14:00), afternoon tea [14:00–17:00), dinner [17:00–20:00), and midnight snack [20:00–04:00), and we statistic numerical features to each period. We construct temporal features by incorporating information such as the Gregorian and lunar dates, holiday occurrences, and weekdays for each day.
Simultaneously, we formulate spatial features from geographical data, encompassing the user's current location, city ID, district ID, and AOI (\textit{Area of Interest}) ID. 
% Embedding Representation
We utilize embedding matrices to semantically encode user attributes, shop attributes, queries, user behaviors, and spatio-temporal attributes respectively. Specifically, the latent embeddings $e_{l_i}$ for attribute features are derived from $[e_{l_o},e_{l_u},e_{l_s}]$, which correspond to cross features embedding, user embedding, and shop embedding, respectively. The query embedding is denoted as $e_q$, the behavior embedding as $e_b$, and the spatio-temporal embedding as $e_{st}$. Notably, in certain marketing campaigns different from user-level modeling, we employ regional modeling for AOIs. Therefore, depending on marketing campaigns, we tailor our features to use either users or AOIs (e.g., cross features, historical click sequences).

%% Temporal Activation
%%%%%%%%%%%%%%%%%%%%%%%%%%%%%%%%%%%%%%%%%%%%%%%%%%%%%%%%%%%%%%%%%%%%%%%%%%%%%%%%%%%%
\textbf{Spatio-temporal Correlation.} We design two spatio-temporal correlation modules, namely \textit{Temporal Activation} and \textit{Temporal Target Attention}, to respectively explore the spatio-temporal characteristics and preferences inherent in static attributes and user behaviors. Within the \textit{Temporal Activation} module, spatio-temporal features input to compute the weights of each feature embedding (i.e., User, Shop, Query) \( e_{l_i} \) through a dot-product operation followed by a sigmoid function. This process aims to capture the dynamic variations inherent in spatio-temporal representations. The outcome is then multiplied by the input representations and integrated via a residual connection, designed to activate the spatio-temporal information encapsulated within the feature embeddings. The output of the Temporal Activation module, denoted as \( h_{ta} \), can be formulated as follows:
\vspace{-0.5em}
\begin{equation}
    h_{ta} = {\rm sigmoid}(\frac{{\rm FC}(e_{st})\cdot e_{l_i}^T}{\sqrt{d_{e_{l_i}}}})e_{l_i} + e_{l_i}.
\end{equation}
%% Temporal Target Attention + Seq
%%%%%%%%%%%%%%%%%%%%%%%%%%%%%%%%%%%%%%%%%%%%%%%%%%%%%%%%%%%%%%%%%%%%%%%%%%%%%%%%%%%%
\vspace{-0.1em}Employing the \textit{Temporal Target Attention} mechanism to catch the interest preferences of different AOIs/Users towards Shops while simultaneously learning the variation across different times and spaces effectively. We accomplish this by applying linear projections to the \textit{Query}, \textit{Key}, and \textit{Value} vectors using weights and biases (e.g., $[w_Q,b_Q]$) generated by fully connected neural networks from $e_{st}$.
The module output denoted as \( h_{tta} \), is articulated as follows:
% \begin{equation}
%     \begin{aligned}
%     Q &= w_Q \cdot e_q + b_Q\\
%     K &= w_K \cdot e_s + b_K\\
%     V &= w_V \cdot e_s + b_V,
%     \end{aligned}
% \end{equation}
\begin{equation}
    h_{tta} = {\rm softmax}(\frac{(w_Q \cdot e_q + b_Q)(w_K \cdot e_s + b_K)^T}{\sqrt{d_k}})(w_V \cdot e_s + b_V).
\end{equation}

%% CMNN
%%%%%%%%%%%%%%%%%%%%%%%%%%%%%%%%%%%%%%%%%%%%%%%%%%%%%%%%%%%%%%%%%%%%%%%%%%%%%%%%%%%%
\textbf{Monotonic Response.} To circumvent the difficulties in training when combining traditional linear monotonicity constraints with saturated (bounded) activation functions
and the limitations that only convex models can be learned by combining convex activation functions, we construct a monotonic layer based on \textit{Constrained Monotonic Neural Networks (CMNN)}\cite{runje2023constrained} framework. This approach not only facilitates utilizing non-saturated activation functions but also allows for a priori modeling of varying degrees of convexity and concavity across diverse application scenarios.

We employ a monotonicity indicator vector $t \in \{-1, 0, 1\}^n$ for the element-wise weights transformation of a feature vector to constrain the monotonicity. The operation $|\cdot|^t$ is applied as follows:
\begin{equation}
    w_{j, i}^{\prime}= \begin{cases}\left|w_{j, i}\right| & \text { if } t_i=1 \\ -\left|w_{j, i}\right| & \text { if } t_i=-1 \\ w_{j, i} & \text { otherwise }\end{cases}
    \label{eq10}
\end{equation}

We rigorously define a convex activation function $\breve{\rho} \in \mathcal{\breve{A}}$, where $\mathcal{\breve{A}}$ denotes the set of all convex, monotonically increasing, lower-bounded, zero-centred functions. 
% \vspace{-1em}
\begin{figure}[b]
\centering
\includegraphics[width=4.5cm]{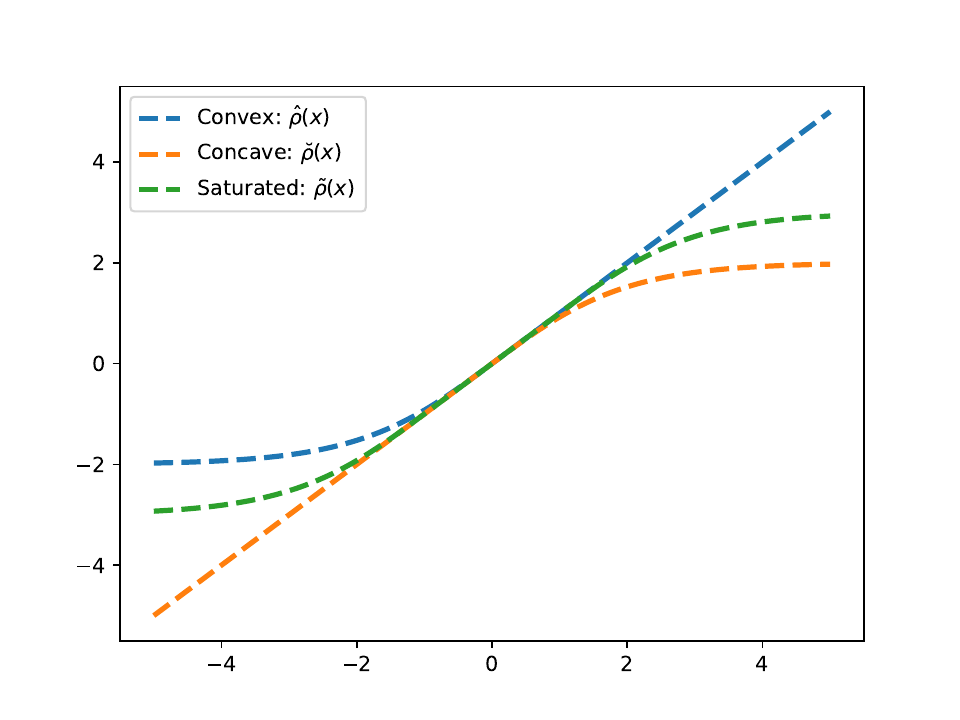}
\caption{CLU based activation functions construction}
\label{fig:clu}
\end{figure}
And we define the concave activation function $\hat{\rho}(x)$ and the saturated function $\tilde{\rho}(x)$ as follows:
\begin{equation}
    \begin{aligned}
& \hat{\rho}(x)=-\breve{\rho}(-x) \\
& \tilde{\rho}(x)= \begin{cases}\breve{\rho}(x+1)-\breve{\rho}(1) & \text { if } x<0 \\
\hat{\rho}(x-1)+\breve{\rho}(1) & \text { otherwise }\end{cases}
\label{m-eq11}
\end{aligned}
\end{equation}

The function $\hat{\rho}(x)$ is concave and upper-bounded, while the function $\tilde{\rho}(x)$ is saturated and bounded. Furthermore, the activation selection vector $\mathbf{s} = (\breve{s}, \hat{s}, \tilde{s}) \in \mathbb{N}^3$ and $\breve{s} + \hat{s} + \tilde{s} = m$ is introduced to approximate the concavity and convexity of the monotonic response through combining concave, convex, and saturated activation functions, where $m$ denotes the monotonic variables dimension.

Specifically, by setting $\mathbf{s}=(m, 0, 0)$, we guarantee that the monotonic response is modeled as a convex function, with analogous principles applied for concave modeling. Then the output of the \textit{combined activation functions} $\rho^\mathbf{s}$ operates element-wise as follows:

\begin{equation}
    \rho^{\mathbf{s}}(\mathbf{h})_j= \begin{cases}\breve{\rho}\left(h_j\right) & \text { if } j \leq \breve{s} \\ \hat{\rho}\left(h_j\right) & \text { if } \breve{s}<j \leq \breve{s}+\hat{s} \\ \tilde{\rho}\left(h_j\right) & \text { otherwise }\end{cases}
\end{equation}

%% CLU
%%%%%%%%%%%%%%%%%%%%%%%%%%%%%%%%%%%%%%%%%%%%%%%%%%%%%%%%%%%%%%%%%%%%%%%%%%%%%%%%%%%%
Moreover, since the user response to incentives is not smooth across the entire treatment interval, it is impractical to model these curves using a linear representation with a fixed smoothing exponential function, as noted by \cite{zhao2019unified}.
Accordingly, to flexibly adapt to the spatiotemporally diverse property of marketing campaigns and learn the monotonic response functions, we introduce a strictly monotonic convex activation function, named \textit{Convex Linear Unit (CLU)}. The detailed proof of the convexity properties of CLU is provided in Appendix \ref{B-2}. Furthermore, we derive the corresponding concave activation function $\hat{\rho}(x)$ and the saturated activation function $\tilde{\rho}(x)$ through CLU. The activation functions construction is depicted in Figure~\ref{fig:clu}, with CLU mathematical expressions as follows:

\begin{equation}
\begin{aligned}
& \breve{\rho}(x)= \begin{cases}-\frac{\omega_0}{2}+\frac{\omega_0}{1+e^{-\omega_1 x}} & \text { if } x<0 \\
x & \text { otherwise }\end{cases} \\
\end{aligned}
\end{equation}

%% FPM
%%%%%%%%%%%%%%%%%%%%%%%%%%%%%%%%%%%%%%%%%%%%%%%%%%%%%%%%%%%%%%%%%%%%%%%%
The \textit{Universal Approximation Theorem}\cite{cybenko1989approximation,hornik1991approximation,pinkus1999approximation} posits that any continuous function over a closed interval can be approximated by neural networks with nonpolynomial activation functions. Concurrently, any multivariate continuous monotonic function on compact subsets of $\mathbb{R}^k$ can be approximated by monotonic constrained neural networks with at most \(k\) hidden layers using sigmoid activation\cite{daniels2010monotone}. Inspired by this, we propose the \textit{Four-Parameter Model (FPM)} through an adaptive translation of the sigmoid, combined with the CMNN framework, to flexibly approximate monotonic response functions in complex scenarios. Meanwhile, we derive Theorem \ref{main-theorem1}, with a detailed proof available in the Appendix \ref{B-1}.

%% Theorem 1
\begin{theorem}
    Let $\breve{\rho} \in \mathcal{\breve{A}}$. Then any multivariate continuous monotone function f on a compact subset of $\mathbb{R}^k$ can be approximated with a monotone constrained neural network of at most $k$ layers using $\rho$ as the activation function.
    \label{main-theorem1}
\end{theorem}

Figure~\ref{fig:four_param_func} elucidates the functional form and parametric behavior of the FPM method,  we learn and interpret the incentive sensitivity curve through four parameters, each bearing tangible significance in marketing. The parameter \( \omega_0 \) represents the natural conversion rate in the absence of any incentive. The parameter \( \omega_1 \) denotes the sensitivity of the user at the motivational inflection point \( \omega_2 \), while \( \omega_3 \) corresponds to the maximal conversion rate attainable by users.

\begin{table*}
  \centering
  \caption{Overall performance on Exploding Red Packets and Delivery Fee Waiver.}
  \vspace{-0.3em}
  \label{offline-performance}
  \resizebox{\textwidth}{!}{
  \begin{tabular}{lcccccccccccccccccccccccccccccc}
  \toprule
    \multirow{2}{*}{\textbf{Method}} & \multicolumn{4}{c}{\textbf{Exploding Red Packets}}& \multicolumn{4}{c}{\textbf{Delivery Fee Waiver}}\\ \cline{2-10} \rule{0pt}{2ex}
                        & \textbf{AUC}\ \raisebox{0.3ex}{$\uparrow$}     & \textbf{MSE}\ \raisebox{0.3ex}{$\downarrow$}   & \textbf{KL Div}\ \raisebox{0.3ex}{$\downarrow$}   & \textbf{Corr Coeff}\ \raisebox{0.3ex}{$\uparrow$}   &
                        \textbf{MAE}\ \raisebox{0.3ex}{$\downarrow$}     & \textbf{MSE}\ \raisebox{0.3ex}{$\downarrow$}   & \textbf{KL Div}\ \raisebox{0.3ex}{$\downarrow$}   & \textbf{Corr Coeff}\ \raisebox{0.3ex}{$\uparrow$}  &\\
    \midrule
    DNN\cite{schmidhuber2015deep}  & 0.6974 & 0.0636 & 0.6243 & 0.8513 & 0.1841 & 0.0831 & 0.8332 & 0.8304 &\\
    DNN-M\cite{archer1993application}  & 0.7022 & 0.0431 & 0.4259 & 0.8713 & 0.1232 & 0.0422 & 0.5347 & 0.8629 &\\
    SBBM\cite{zhao2019unified}  & 0.7081 & 0.0343 & 0.2110 & 0.8927 & 0.1079 & 0.0338 & 0.3219 & 0.8827 &\\
    FPM  & 0.7113 & 0.0293 & 0.1901 & 0.9066 & 0.0913 & 0.0323 & 0.3083 & 0.8961 & \\
    CMNN\cite{runje2023constrained}(ReLU\cite{nair2010rectified})  & 0.7111 & 0.0282 & 0.1848 & 0.9150 & 0.0894 & 0.0298 & 0.2772 & 0.8916 & \\
    CMNN(ELU\cite{clevert2015fast})  & 0.7139 & 0.0273 & 0.1737 & 0.9171 & 0.0823 & 0.0281 & 0.2761 & 0.9127 &\\
    CMNN(CLU)  & 0.7151 & 0.0269 & 0.1718 & 0.9189 & 0.0809 & 0.0274 & 0.2343 & 0.9143 & \\
    CoMAN-B  & 0.7348 & 0.0224 & 0.1305 & 0.9229 & 0.0769 & 0.0263 & 0.2218 & 0.9193 & \\
    CoMAN w/o AA   & 0.7414 & 0.0191 & 0.0984 & 0.9333 & 0.0624 & 0.0227 & 0.1897 & 0.9326 & \\
    CoMAN w/o S-t  & 0.7407 & 0.0159 & 0.0925 & 0.9357 & 0.0733 & 0.0243 & 0.1809 & 0.9273 & \\
    % \midrule
    \multirow{2}{*}{\textbf{CoMAN}} & \textbf{0.7480} & \textbf{0.0115} & \textbf{0.0901} & \textbf{0.9486} & \textbf{0.0454} & \textbf{0.0184} & \textbf{0.1211} & \textbf{0.9469} &\\
                        & \textcolor{deepred}{\textbf{+0.0506}}   & \textcolor{deepred}{\textbf{+0.0521}} & \textcolor{deepred}{\textbf{+0.5342}}   & \textcolor{deepred}{\textbf{+0.0973}}& \textcolor{deepred}{\textbf{+0.1387}} & \textcolor{deepred}{\textbf{+0.0647}} &\textcolor{deepred}{\textbf{+0.7121}}& \textcolor{deepred}{\textbf{+0.1165}}   & \\
    \bottomrule
  \end{tabular}}
\end{table*}
% \begin{table}[!htbp]
%   \caption{Exploding Red Packets}
%   \label{tab:cross-domain}
%   \centering
%   \begin{tabular}{lcccccc}
%   \toprule
%     \textbf{Method} & \textbf{MSE} & \textbf{MAPE} & \textbf{KL Div} & \textbf{Corr Coeff}\\
%     \midrule
%     Four Param Model & 0.149 & 1.150 & 2.551 & 0.690 &\\
%     \textbf{CoMAN Mono(ours)} & \textbf{0.019} & \textbf{0.347} & \textbf{0.520} & \textbf{0.829} &\\
%     \bottomrule
%   \end{tabular}
% \end{table}
\begin{table*}
  \centering
  \caption{Temporal segment evaluation on Exploding Red Packets. We denote the following periods: Bkfst (Breakfast), Lch (Lunch), A.Tea (Afternoon Tea), Din (Dinner), and M.S. (Midnight Snack).}
  \vspace{-0.3em}
  \label{online-erp}
  \resizebox{\textwidth}{!}{
  \begin{tabular}{lcccccccccccccccccccccccccccccccc}
  \toprule
    \multirow{2}{*}{\textbf{Method}} & \multicolumn{6}{c}{\textbf{MAE\ \raisebox{0.3ex}{$\downarrow$}}}& \multicolumn{6}{c}{\textbf{KL Div\ \raisebox{0.3ex}{$\downarrow$}}}\\ \cline{2-14} \rule{0pt}{2ex}
                        & \textbf{Bkfst}     & \textbf{Lch}   & \textbf{A.Tea}   & \textbf{Din}  & \textbf{M.S.} & \textbf{Overall}&
                        \textbf{Bkfst}     & \textbf{Lch}   & \textbf{A.Tea}   & \textbf{Din}  & \textbf{M.S.}& \textbf{Overall}&\\
    \midrule
    FPM  & 0.1782 & 0.1385 & 0.1461 & 0.1297 & 0.1118 & 0.1325 & 0.1395 & 0.0824 & 0.1024 & 0.0918 & 0.0439 & 0.2801& \\
    % CoMAN w/o (S-t + AA)  & 0.0392 & 0.0715 & 0.0620 & 0.0814 & 0.0660 & 0.0495 & 0.0493 & 0.0729 & 0.0739 & 0.0882 & 0.0412 & 0.2236 & \\
    % CoMAN w/o AA  & 0.0955 & 0.0553 & 0.0672 & 0.0568 & 0.0374 & 0.0574 & 0.0313 & 0.0428 & 0.0901 & 0.0457 & 0.0142 & 0.2181 & \\
    % CoMAN w/o S-t  & 0.0917 & 0.0813 & 0.0553 & 0.0861 & 0.0590 & 0.0464 & 0.0377 & 0.1177 & 0.0269 & 0.1197 & 0.0380 & 0.1569 & \\
    % \midrule
    \multirow{2}{*}{\textbf{CoMAN}} & \textbf{0.0369} & \textbf{0.0449} & \textbf{0.0529} & \textbf{0.0780} & \textbf{0.0629} & \textbf{0.0448} & \textbf{0.0297} & \textbf{0.0288} & \textbf{0.0432} & \textbf{0.0737} & \textbf{0.0360} & \textbf{0.1480} &\\
                        & \textcolor{deepred}{\textbf{+0.1413}}   & \textcolor{deepred}{\textbf{+0.0936}} & \textcolor{deepred}{\textbf{+0.0932}}   & \textcolor{deepred}{\textbf{+0.0517}}& \textcolor{deepred}{\textbf{+0.0489}} & \textcolor{deepred}{\textbf{+0.0877}} &\textcolor{deepred}{\textbf{+0.1098}}& \textcolor{deepred}{\textbf{+0.0536}}   & \textcolor{deepred}{\textbf{+0.0592}}& \textcolor{deepred}{\textbf{+0.0181}}  & \textcolor{deepred}{\textbf{+0.0079}} & \textcolor{deepred}{\textbf{+0.1321}} &  \\
    \bottomrule
  \end{tabular}}
\end{table*}
\begin{table*}
  \centering
  \caption{Temporal segment evaluation on Delivery Fee Waiver}
  \vspace{-0.3em}
  \label{online-dfw}
  \resizebox{\textwidth}{!}{
  \begin{tabular}{lcccccccccccccccccccccccccccccccc}
  \toprule
    \multirow{2}{*}{\textbf{Method}} & \multicolumn{6}{c}{\textbf{MAE\ \raisebox{0.3ex}{$\downarrow$}}}& \multicolumn{6}{c}{\textbf{KL Div\ \raisebox{0.3ex}{$\downarrow$}}}\\ \cline{2-14} \rule{0pt}{2ex}
                        & \textbf{Bkfst}     & \textbf{Lch}   & \textbf{A.Tea}   & \textbf{Din}  & \textbf{M.S.} & \textbf{Overall}&
                        \textbf{Bkfst}     & \textbf{Lch}   & \textbf{A.Tea}   & \textbf{Din}  & \textbf{M.S.}& \textbf{Overall}&\\
    \midrule
    FPM  & 0.1073 & 0.0883 & 0.1021 & 0.0932 & 0.0972 & 0.0943 & 0.0721 & 0.0809 & 0.0678 & 0.0733 & 0.0749 & 0.3692 & \\
    % \midrule
    \multirow{2}{*}{\textbf{CoMAN}} & \textbf{0.0771} & \textbf{0.0569} & \textbf{0.0732} & \textbf{0.0631} & \textbf{0.0654} & \textbf{0.0632} & \textbf{0.0307} & \textbf{0.0276} & \textbf{0.0283} & \textbf{0.0259} & \textbf{0.0271} & \textbf{0.1398} &\\
                        & \textcolor{deepred}{\textbf{+0.0302}}   & \textcolor{deepred}{\textbf{+0.0314}} & \textcolor{deepred}{\textbf{+0.0289}}   & \textcolor{deepred}{\textbf{+0.0301}}& \textcolor{deepred}{\textbf{+0.0318}} & \textcolor{deepred}{\textbf{+0.0311}} &\textcolor{deepred}{\textbf{+0.0414}}& \textcolor{deepred}{\textbf{+0.0533}}   & \textcolor{deepred}{\textbf{+0.0395}}& \textcolor{deepred}{\textbf{+0.0474}}  & \textcolor{deepred}{\textbf{+0.0478}} & \textcolor{deepred}{\textbf{+0.2294}} &  \\
    \bottomrule
  \end{tabular}}
\end{table*}

% \begin{table}[!htbp]
%   \caption{Inflection points}
%   \label{tab:inflection-points}
%   \centering
%   \begin{tabular}{lccccccc}
%   \toprule
%     \textbf{Model} & \textbf{Bkfst}     & \textbf{Lch}   & \textbf{A.Tea}   & \textbf{Din}  & \textbf{M.S.} & \textbf{Overall}&\\
%     \midrule
%     Baseline(FPM) & +0.82 &	+0.102 &	+0.454 \\
%     CoMAN w/o (S-t + AA) & +0.82\% &	+0.102\% &	+0.454\% \\
%     CoMAN w/o AA &	+1.198\% &	+0.506\% &	+1.564\% \\
%     CoMAN w/o S-t &	+2.177\% &	+3.015\% &	+3.565\% \\
%     CoMAN & & &			 \\
%     \bottomrule
%   \end{tabular}
% \end{table}

% \begin{table*}[!htbp]
%   \caption{avg Delivery Fee Waiver}
%   \label{tab:cross-domain}
%   \centering
%   \begin{tabular}{lcccccc}
%   \toprule
%     \textbf{Method}& \textbf{MAE\ \raisebox{0.3ex}{$\downarrow$}} & \textbf{MSE\ \raisebox{0.3ex}{$\downarrow$}} & \textbf{MAPE\ \raisebox{0.3ex}{$\downarrow$}} & \textbf{KL Div\ \raisebox{0.3ex}{$\downarrow$}}& \textbf{Corr Coeff\ \raisebox{0.3ex}{$\uparrow$}}\\
%     \midrule
%     Four Param Model & 0.094 & 0.032 & 0.545 & 0.369 & 0.896 &\\
%     \multirow{2}{*}{\textbf{CoMAN(ours)}} & \textbf{0.063} & \textbf{0.017} & \textbf{0.238} & \textbf{0.140} & \textbf{0.946} & \\
%                         & \textcolor{deepred}{\textbf{+0.031}}   & \textcolor{deepred}{\textbf{+0.015}}& \textcolor{deepred}{\textbf{+0.307}}   & \textcolor{deepred}{\textbf{+0.229}}& \textcolor{deepred}{\textbf{+0.050}}  & \\
%     \bottomrule
%   \end{tabular}
% \end{table*}

%% Adaptive Activation
%%%%%%%%%%%%%%%%%%%%%%%%%%%%%%%%%%%%%%%%%%%%%%%%%%%%%%%%%%%%%%%%%%%%%%%%%%%%%%%%%%%%
\textbf{Adaptive Response Enhancement.} To enable the model to learn the variability of monotonic response across different spatio-temporal dimensions, we design the  \textit{Adaptive Activation} and \textit{Spatio-temporal Attention (S-t Attention)} modules within the monotonic layer to enhance the model's capability to perceive spatio-temporal information. These modules are designed to flexibly adapt the convexity, concavity, and parametric expressions of monotonic response functions through spatio-temporal representations.

In previous CMNN methodology, $\mathbf{s}$ is a priori parameter requiring manual adjustment through tuning rather than allowing for adaptive learning.
The inflection point of a monotonic response function divides it into convex and concave segments. Drawing on these insights, we propose an Adaptive Activation module capable of adaptively learning and adjusting the proportion of convexity and concavity by catching and fusing deep representations with spatio-temporal information. As depicted on the left of the monotonic layer in Figure~\ref{fig:coman}, we mask different parts of the monotonic representation \(m_i\) with the adaptively learned convexity-concavity ratio factor \(s_{st}\) and then combine through convex, concave, and saturated activation functions separately. The operations of the Adaptive Activation can be elucidated as follows:
\begin{equation}
    s_{st} = |{\rm softmax}(\frac{m_i{({\rm Concat}(r_f, e_{st})})^T}{\sqrt{d_k}})m_i|.
\end{equation}
\begin{figure}[t]
\centering
\includegraphics[width=5.5cm]{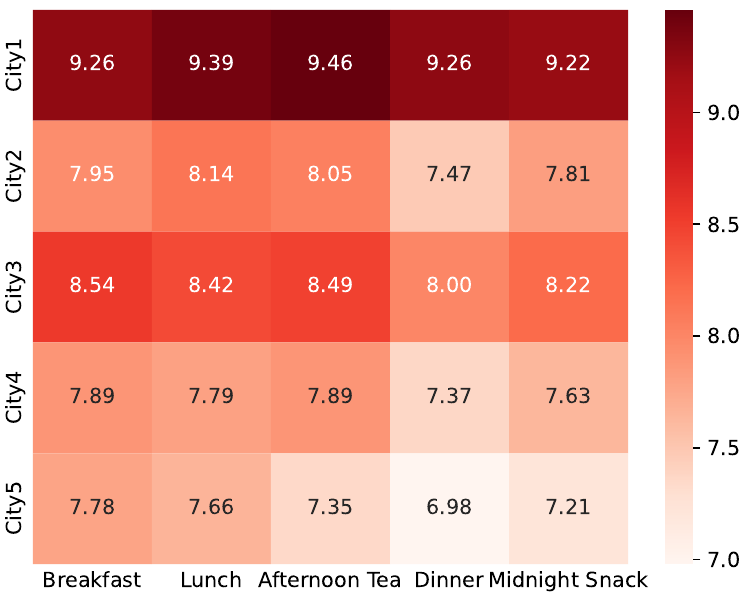}
\caption{Visualization of the average redemption amount of coupons in the top five cities with the most users across different periods.}
\vspace{-2em}
\label{fig:amount_distribution}
\end{figure}

%% S-t Attention
%%%%%%%%%%%%%%%%%%%%%%%%%%%%%%%%%%%%%%%%%%%%%%%%%%%%%%%%%%%%%%%%%%%%%%%%%%%%%%%%%%%%
We design an \textit{Spatio-temporal (S-t) Attention} module to utilize spatio-temporal information to further adaptively adjust and learn FPM. Firstly, we generate weights and bias ($[w_i, b_i] $) terms for deep fused representation vectors (e.g., users, shops, etc.) through a fully connected layer. Then we conduct matrix multiplication calculations to derive the FPM parameters. The module structure is illustrated in Figure~\ref{fig:module_1}, and can be formulated as follows:
\vspace{-0.3em}
\begin{equation}
    \omega_i = {\rm FC}(w_i \cdot {\rm FC}(r_f) + b_i).
\end{equation}

%%%%%%%%%%%%%%%%%%%%%%%%%%%%%%%%%%%%%%%%%%%%%%%%%%%%%%%%%%%%%%%%%%%%%%%%
% \section{Evaluation}
% \label{Evaluation}

% Simulation (algorithm evaluation)
%%%%%%%%%%%%%%%%%%%%%%%%%%%%%%%%%%%%%%%%%%%%%%%%%%%%%%%%%%%%%%%%%%%%%%%%

\section{Experiments}
\label{Experiments}

\subsection{Experimental Setup}
\label{Experimental Setup}

\begin{figure*}[!htbp]
\centering
\begin{subfigure}[b]{0.25\textwidth}
  \includegraphics[width=\linewidth]{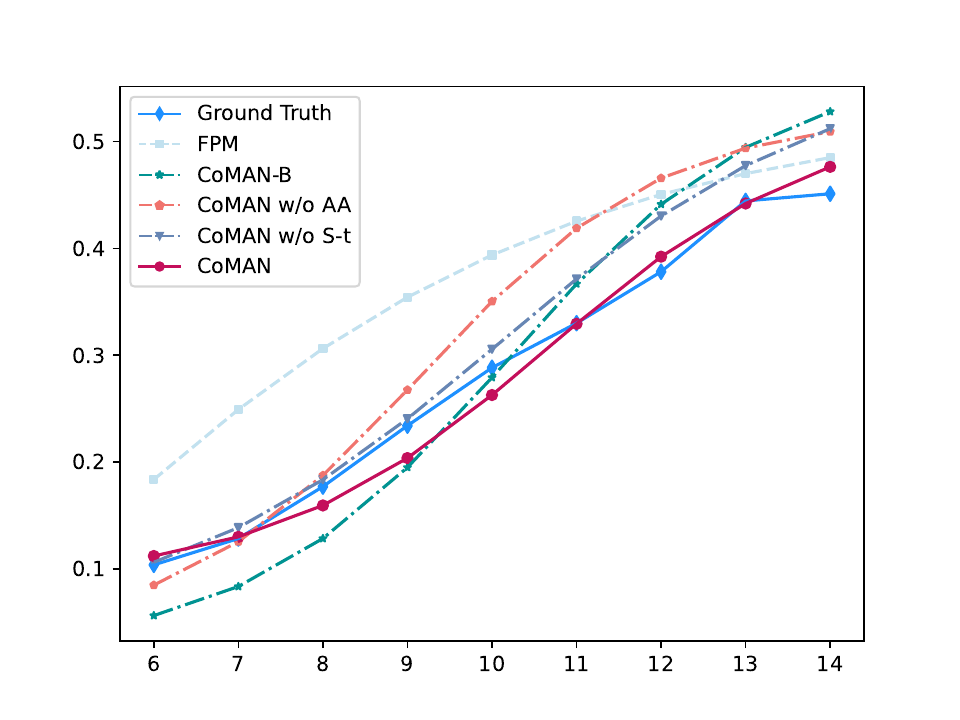}
  \caption{Overall}
  \label{fig:baseline}
\end{subfigure} % \quad provides some space between the subfigures
\begin{subfigure}[b]{0.25\textwidth}
  \includegraphics[width=\linewidth]{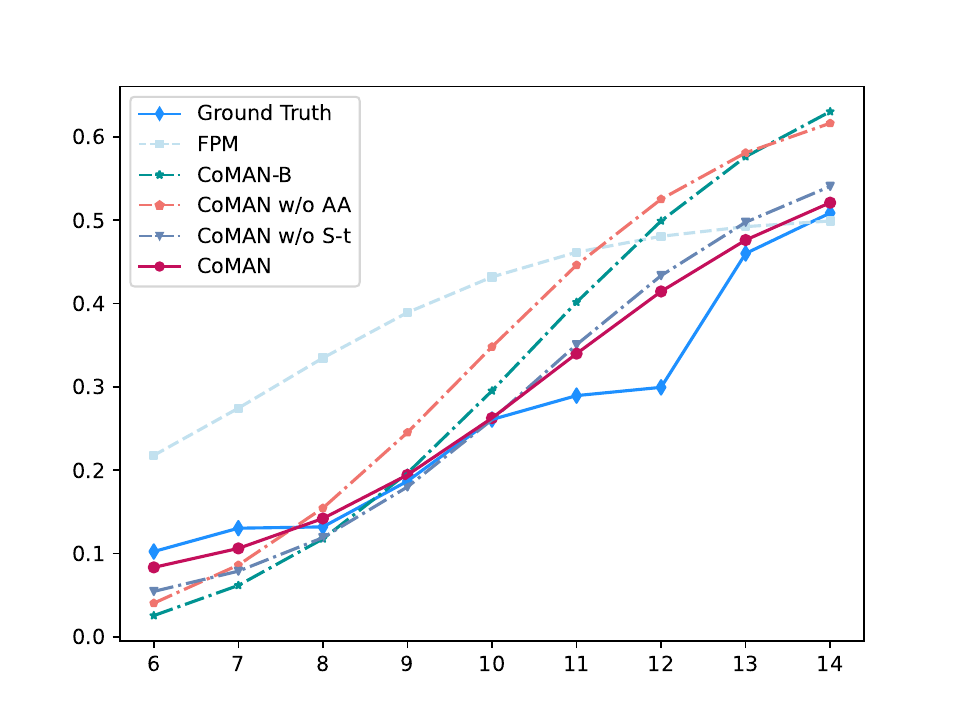}
  \caption{Breakfast}
  \label{fig:cmtn_base}
\end{subfigure}
\begin{subfigure}[b]{0.25\textwidth}
  \includegraphics[width=\linewidth]{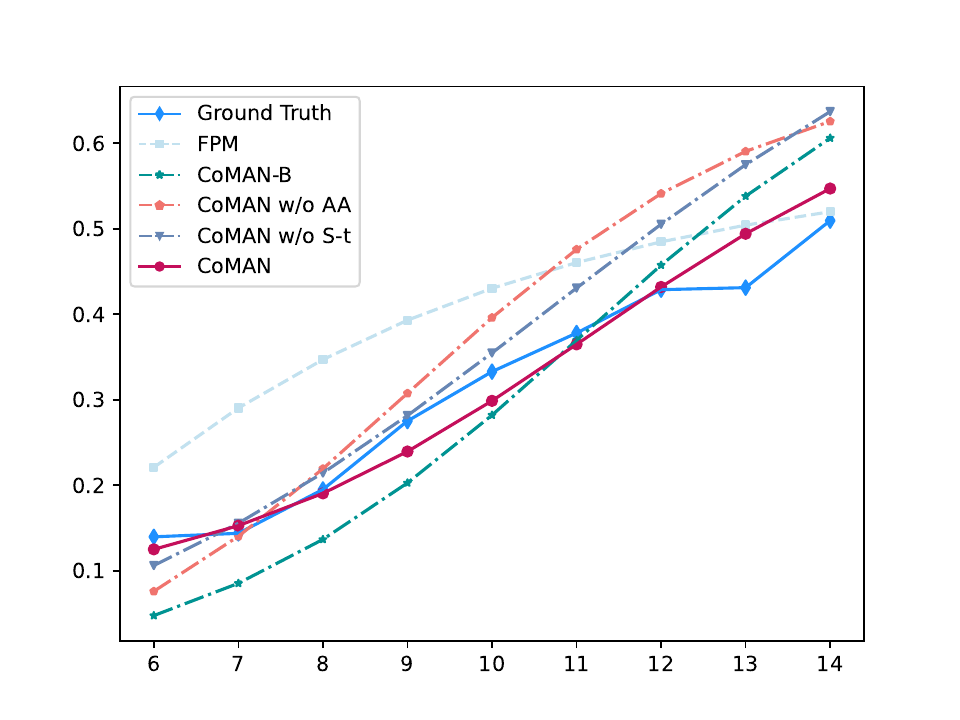}
  \caption{Lunch}
  \label{fig:cmtn_base}
\end{subfigure}

\begin{subfigure}[b]{0.25\textwidth}
  \includegraphics[width=\linewidth]{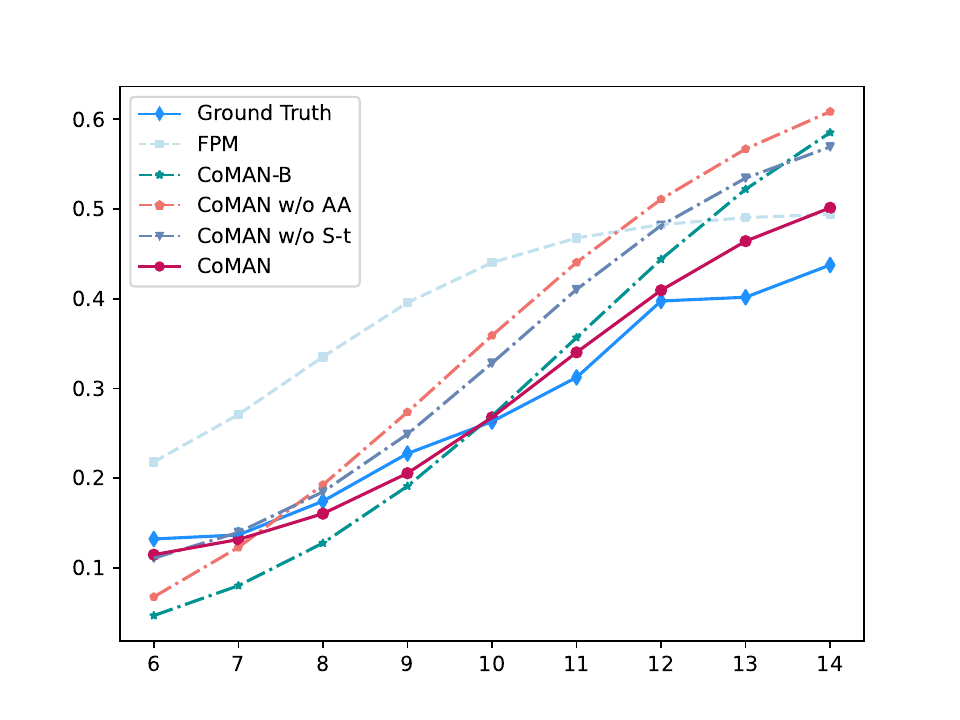}
  \caption{Afternoon Tea}
  \label{fig:st_att}
\end{subfigure}
\begin{subfigure}[b]{0.25\textwidth}
  \includegraphics[width=\linewidth]{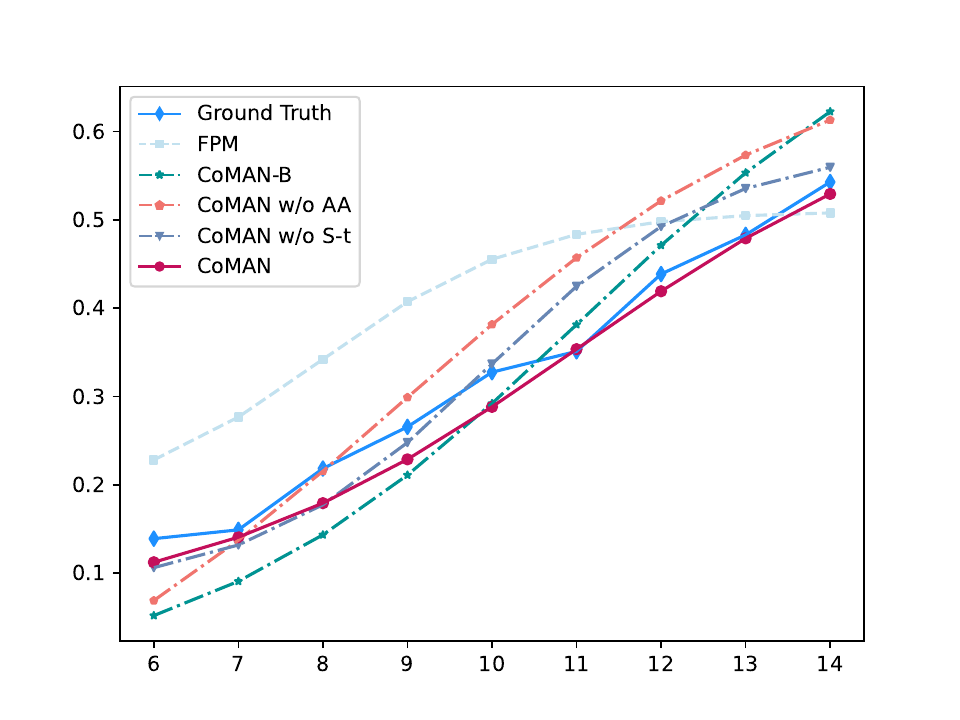}
  \caption{Dinner}
  \label{fig:adapt_act}
\end{subfigure}
\begin{subfigure}[b]{0.25\textwidth}
  \includegraphics[width=\linewidth]{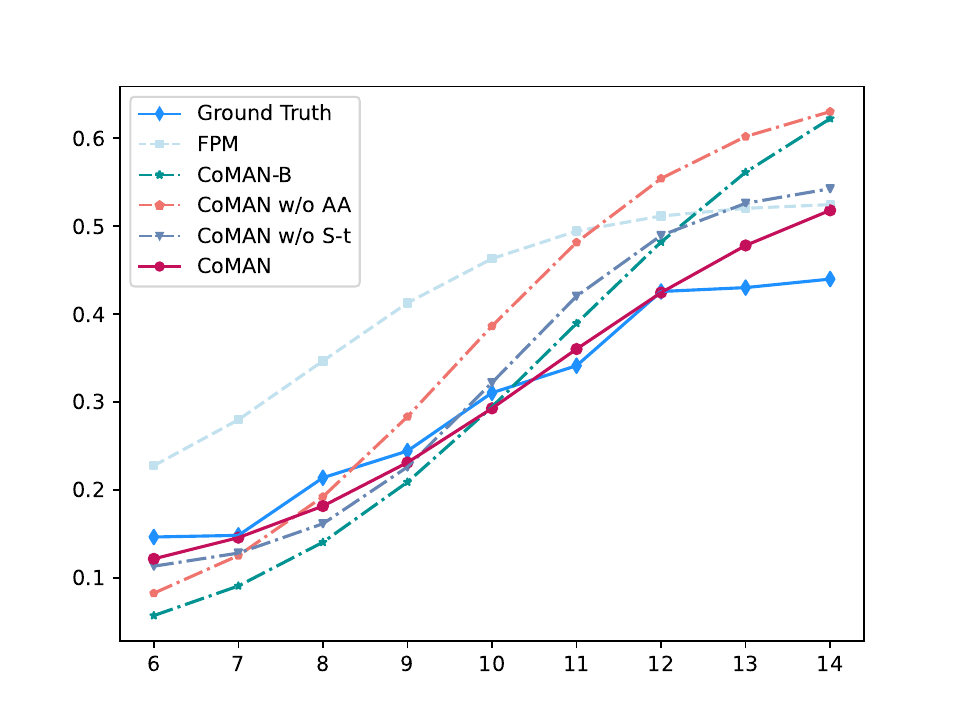}
  \caption{Midnight Snack}
  \label{fig:cmtn_fv}
\end{subfigure}

\begin{subfigure}[b]{0.25\textwidth}
  \includegraphics[width=\linewidth]{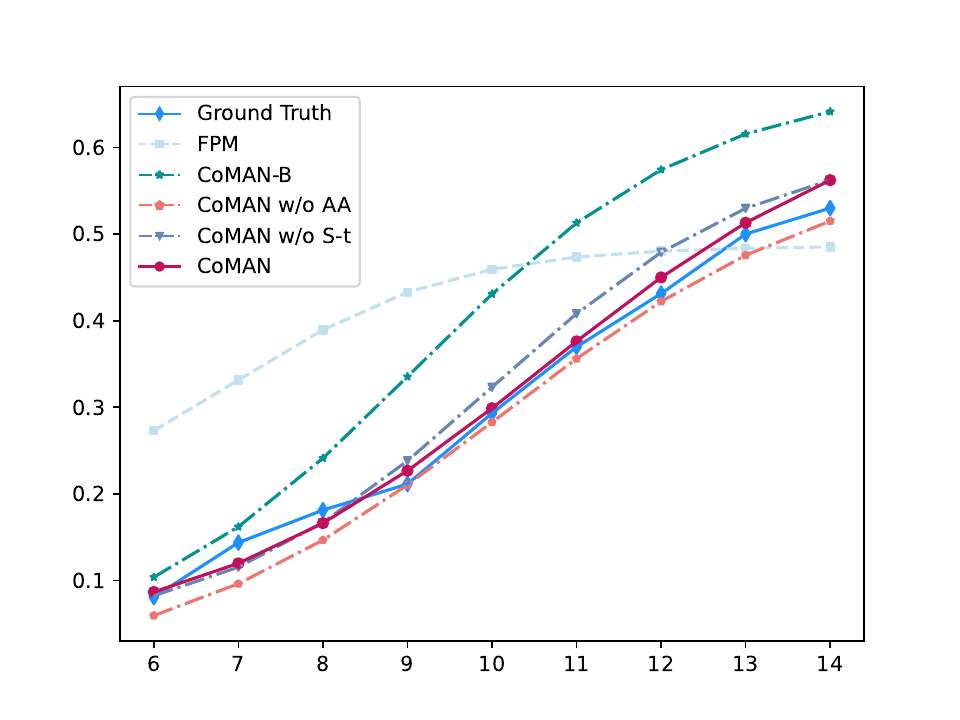}
  \caption{City 1}
  \label{fig:city1}
\end{subfigure}
\begin{subfigure}[b]{0.25\textwidth}
  \includegraphics[width=\linewidth]{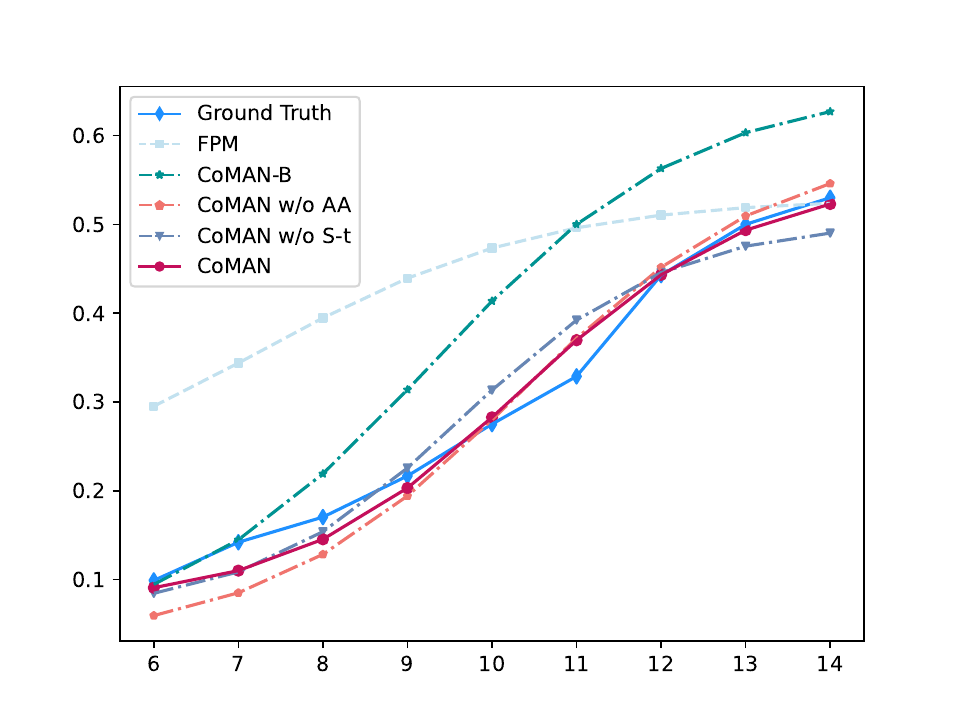}
  \caption{City 2}
  \label{fig:city2}
\end{subfigure}
\begin{subfigure}[b]{0.25\textwidth}
  \includegraphics[width=\linewidth]{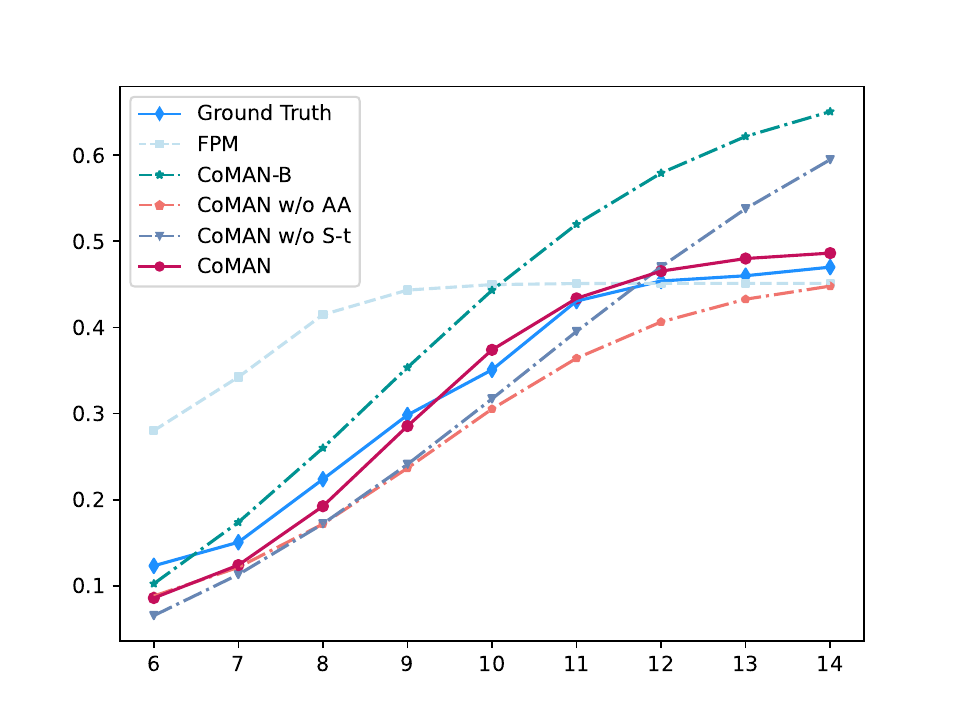}
  \caption{City 3}
  \label{fig:city3}
\end{subfigure}
\vspace{-0.5em}
\caption{Visualization of the prediction scores for incentive responses from various online models and the corresponding ground truth. Results are observed overall, with concurrent evaluations segmented by periods and cities.}
\vspace{-1em}
\label{fig:overall-distribution}
\end{figure*}
% \begin{figure*}[h]
% \centering
% \includegraphics[width=18.5cm]{exps/time_exps_total.pdf}
% \caption{time exps total}
% \label{fig:time_exps_total}
% \end{figure*}

\textbf{Dataset Description}. We conducted experiments on real-world datasets\footnote{We gathered two real-world datasets from Ele.me's online platform due to the absence of available public datasets suitable for evaluating our task.} from two marketing campaigns on Ele.me, one of the largest OFOS platforms in China. The training datasets span one week, while the test dataset covers a single day. Notably, we randomly select 5\% of online traffic flow as unbiased sample to randomly allocate incentive values as the ground truth for the incentive sensitivity function approximation. 
(i) Exploding Red Packets Dataset: Send an average of 10 million exploding red packets to users online per day. It's a classification dataset with approximately 1.05 million daily training samples and nearly 400 features, ultimately reaching 8.4 million samples. 
(ii) Delivery Fee Waiver Dataset: It is a regression dataset covering approximately 3.2 million AOIs, with about 4.3 million shops and 303 features. The daily samples are around 770 million, with the total exceeding 6 billion samples. We segment application data by period and city and visualize users' average redemption amount in exploding red packets, as depicted in Figure \ref{fig:amount_distribution}. The visualization reveals that users' preferences for coupon amounts differ across various spatio-temporal settings, which is consistent with the analysis.

\textbf{Evaluation Metric}. To comprehensively evaluate the performance of our model, we assess it from both accuracy and distribution perspectives. We first employ metrics such as the Area Under the Curve (AUC)\cite{fawcett2006introduction}, Mean Absolute Error (MAE)\cite{willmott2005advantages} and Mean Squared Error (MSE)\cite{hyndman2006another} to evaluate the precision of the model. Subsequently, we use the Kullback-Leibler Divergence (KL Div)\cite{kullback1951information} and Correlation Coefficients (Corr Coeff)\cite{pearson1895vii} to assess the similarity between the model's predicted incentive sensitivity functions and the random traffic (i.e. ground truth).

\textbf{Baselines}. To verify the effectiveness of our proposed methodology, we compare it against baselines with several \textit{state-of-the-art} methods as follows: (i) \textit{\textbf{DNN}}: This is a naive model devoid of spatio-temporal and monotonic modeling capabilities. (ii) \textit{\textbf{DNN-M}}: This method constrains the network weights to be exclusively non-negative or non-positive and combines sigmoid activation functions for monotonic modeling. 
(iii) \textit{\textbf{SBBM}}: The Semi-black-box model\cite{zhao2019unified} predicts dynamic marketing responses by enhancing the logit demand curve through neural networks.
(iv) \textit{\textbf{FPM}}: This approach utilizes the mathematical four-parameter function to independently approximate the monotonic response function. (v) \textit{\textbf{CMNN}}: This method employs the CMNN framework in conjunction with different base convex activation functions, including CLU proposed in our study, while combining sigmoid as the final activation function.

\textbf{Hyperparamters}. In our experiments, for Exploding Red Packets, the model training batch size is set to 128 with Adagrad as an optimizer and a learning rate of 0.001. In Delivery Fee Waiver, the batch size is configured to 1024, using an Adagrad optimizer with a learning rate of 0.012. Furthermore, all experiments are conducted on AOP\footnote{A model training platform self-developed by the Alibaba Group} using 40 parameter servers and 400 worker threads. Further implementation details concerning the model parameters can be found in Appendix \ref{appendix-c}.

\begin{figure}[b]
\centering
\vspace{-1.5em}
\begin{subfigure}[b]{0.23\textwidth}
  \includegraphics[width=\linewidth]{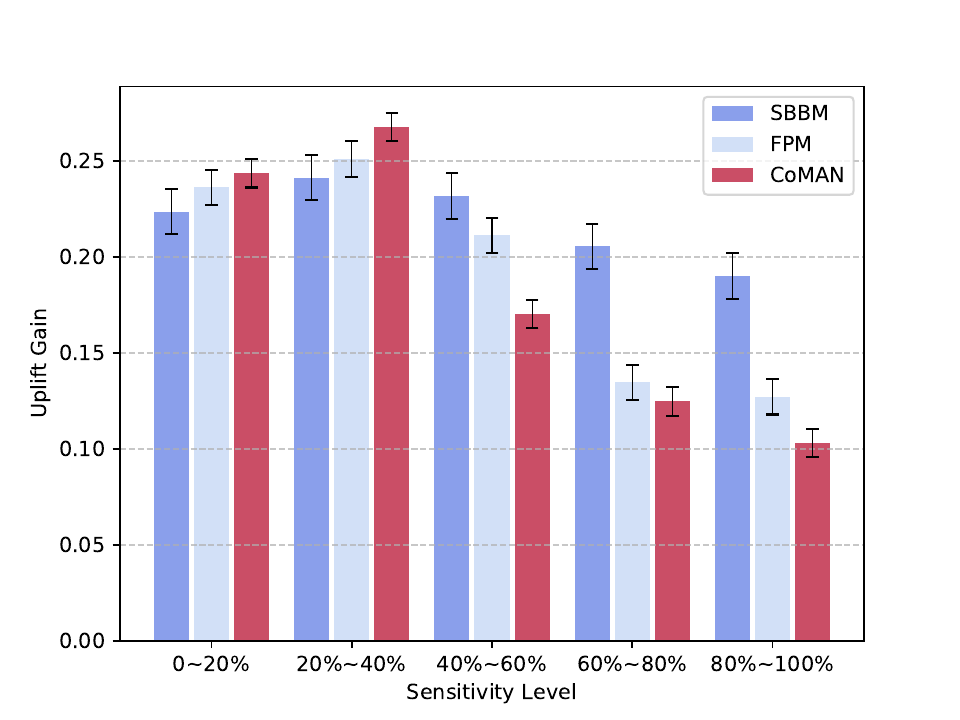}
  \caption{Exploding Red Packets}
  \label{fig:erp_uplift}
\end{subfigure} % \quad provides some space between the subfigures
\begin{subfigure}[b]{0.23\textwidth}
  \includegraphics[width=\linewidth]{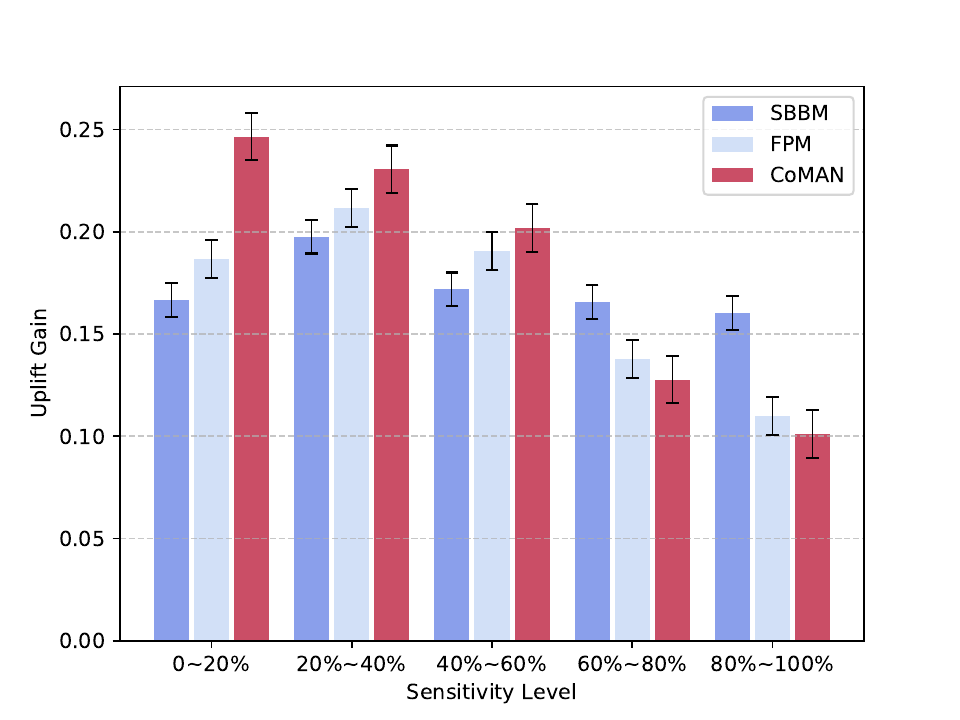}
  \caption{Delivery Fee Waiver}
  \label{fig:dfw_uplift}
\end{subfigure}
\vspace{-0.8em}
\caption{Comparisons of uplift gain for users vary with different incentive sensitivity levels.}
\label{fig:uplift}
\end{figure}

\vspace{-1em}
%%%%%%%%%%%%%%%%%%%%%%%%%%%%%%%%%%%%%%%%%%%%%%%%%%%%%%%%%%%%%%%%%%%%%%%%
\subsection{Experimental Results}
\label{Experimental Results}
%% overall performance
\textbf{Performance Comparison} We report the overall performance of various competitive baselines and our proposed model across two datasets in Table \ref{offline-performance}. Moreover, to evaluate the real-world effectiveness of our approach in practical applications, we present the performance of the model in online marketing campaigns as shown in Table~\ref{online-erp} and Table~\ref{online-dfw}. The results yield the following observations: (i) The performance of CoMAN is notably superior to all competitive baselines. Compared to previous marketing response approaches, CoMAN achieves an average improvement of 5.33\% in user redemption response accuracy and an average enhancement of 5.89\% in distribution similarity. These remarkable performance gains strongly indicate that CoMAN has the potential to enhance marketing efficiency in budget allocation significantly. (ii) Methods like FPM and SBBM, which employ nonlinear mathematical functions in conjunction with neural networks to adaptively enhance the logit response curve, are more effective than simply applying positive or negative constraint learning to neural network weights. These methods are capable of more accurately characterizing users' sensitivity to incentives. Specifically, FPM is more effective than SBBM in representing this sensitivity, thereby highlighting the effectiveness of our approach in integrating FPM for learning monotonic response. (iii) Compared to traditional marketing approaches like DNN-M and SBBM, CoMAN performs better in perceiving and incorporating spatio-temporal information. Furthermore, our model surpasses previous monotonic methods such as CMNN and FPM in effectively capturing spatio-temporal representations which enhances monotonic modeling. These findings underscore the significant advancements of CoMAN, primarily attributed to the activation modeling of spatio-temporal information and the robust capture of spatio-temporal monotonic patterns. (iv) Our designed CLU function, when integrated with CMNN, outperforms the other two activation functions, ReLU and ELU. This outcome substantiates the validity of approximating user monotonic response curves with a non-fixed exponential function, thus better capturing and expressing user sensitivity to incentives. (v) Our approach exhibits superior performance across datasets with different modeling granularities (i.e., user-level and regional-level). This indicates that our method is capable of more precisely describing monotonic response under various incentives, whether the focus is on user-level perception or regional-level modeling.
\begin{figure}[b]
\centering
\vspace{-1.5em}
\begin{subfigure}[b]{0.23\textwidth}
  \includegraphics[width=\linewidth]{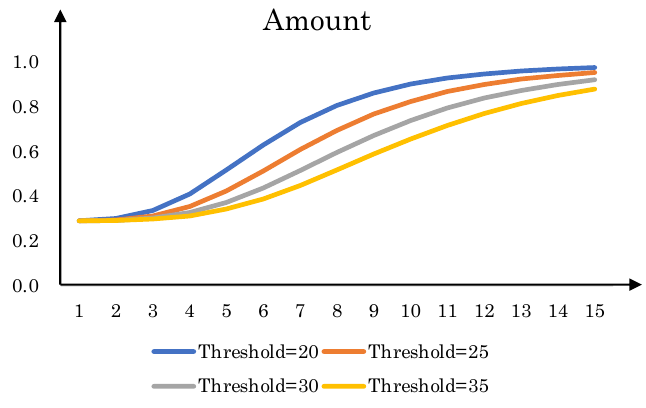}
  \caption{Amount}
  \label{fig:amount}
\end{subfigure} % \quad provides some space between the subfigures
\begin{subfigure}[b]{0.23\textwidth}
  \includegraphics[width=\linewidth]{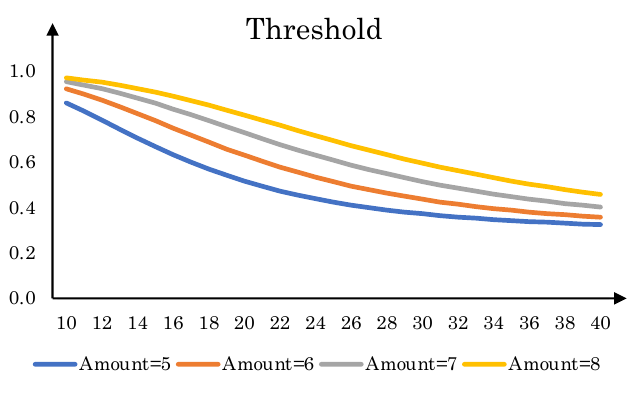}
  \caption{Threshold}
  \label{fig:threshold}
\end{subfigure}
\caption{Monotonic response curves of user sensitivity to incentives with different amounts and thresholds}
\label{fig:main}
\end{figure}

%% Figure 5
To visualize the effectiveness of our model in budget allocation compared to previous marketing methods, we follow the idea proposed by \cite{yu2021joint,liu2019graph} and use the first derivative of the response curve at the inflection point as the "\textit{gradient}" representing a user's sensitivity to incentives. Under the same incentive treatment, user groups with high "\textit{gradient}" values yield greater commercial returns compared to those with low "\textit{gradient}". Therefore, the efficiency of the marketing model can be attributed to the accurate estimation of response scores for users with varying incentive sensitivities. Subsequently, we rank the users in descending order based on their "\textit{gradient}" and divide them into five groups representing different levels of sensitivity to incentives. As illustrated in Figure \ref{fig:uplift}, CoMAN predicts a higher uplift gain for the first group of highly sensitive users compared to SBBM and FPM, while predicting a correspondingly lower uplift gain for the least sensitive users. This demonstrates our model's higher efficiency in budget allocation.

\textbf{Ablation Study}
To comprehensively understand the capability of our proposed CoMAN model in approximating incentive sensitivity functions, and recognizing that the S-t attention and adaptive activation modules constitute the core modeling components of CoMAN, we carried out exhaustive ablation studies to evaluate the effectiveness of these pivotal modules. Accordingly, we introduced three ablation variants: (i) CoMAN w/o (S-t and AA) (CoMAN-B), this baseline version employs CMNN based on CLU combined with FPM as the final activation function.
(ii) CoMAN w/o AA, this revision excludes the Adaptive Activation module from the CoMAN model.
(iii) CoMAN w/o S-t, this variant omits the S-t Attention module from the CoMAN structure. These versions are evaluated against the complete CoMAN model, which integrates all the modules, to systematically assess the impact of each module on the overall model performance. Additionally, the comparative experiments of CMNN without the spatio-temporal activation and the spatio-temporal target attention modules.

To validate the monotonic response property of the CoMAN, we conduct offline tests in which all non-monotonic features are held constant, while the monotonic features (e.g., amount and threshold) increase. Then construct the CVR monotonic response curves, as illustrated in Figures~\ref{fig:amount} and \ref{fig:threshold} respectively.
%% ablation observation
From the results shown in the above tables, we have the following observations:
(i) As our analysis suggests, models that leverage spatio-temporal information in the monotonic layer to enhance the perception of the monotonic response outperform models that consider only basic attributes and behavioral representations (e.g., FPM, CMNN) or incorporate spatio-temporal capture solely in the embedding (e.g., CoMAN-B). (ii) The performance of single-module models in the monotonic layer is significantly lower than that of models employing a two-module perception framework. This finding is consistent with our design rationale, indicating that two-module perception of spatio-temporal information more effectively captures the additional semantics inherent in user monotonic responses. (iii) As anticipated, employing the Adaptive Activation module facilitates more adaptive learning of the concavity and convexity proportions of the response function by harnessing spatio-temporal information. Concurrently, utilizing the S-t Attention module enables a more precise spatio-temporal representation of key parameters in the sensitive response curves.

%% visual distribution
Moreover, as observed in Figure~\ref{fig:overall-distribution}, we evaluate the model's online predicted monotonic response curves both overall and across five periods, as well as for the top three cities with the most users. The fitting results for both the ablated and complete versions of CoMAN more closely approximate the ground truth distribution compared to the baseline, further verifying CoMAN's effective capture of spatio-temporal characteristics in real-world marketing campaigns.
%%%%%%%%%%%%%%%%%%%%%%%%%%%%%%%%%%%%%%%%%%%%%%%%%%%%%%%%%%%%%%%%%%%%%%%%

\textbf{Online A/B Test}
We deploy our methodology to the Ele.me and conduct a two-week online A/B test. Compared with the online baseline model, our proposed model can bring more benefits while maintaining platform subsidies unchanged or even reduced, helping the marketing business improve budget efficiency and bring more growth. Detailed online outcomes as shown in Table~\ref{tab:online_metrics}.

Notably, within the regional pricing strategy for the delivery fee waiver marketing campaign, the CoMAN model also facilitates substantial business benefits for the platform. This is achieved even with a reduction in subsidy intensity by 0.76\%, leading to a 0.70\% increase in orders growth, a 0.62\% improvement in the CTCVR, and a 1.06\% increment in GMV.
\vspace{-1em}
\begin{table}[!htbp]
  \caption{Online Business Performance Metrics. In comparison to the online FPM method, the CoMAN model demonstrated an improvement rate in the marketing of Ele.me.}
  \label{tab:online_metrics}
  \centering
  \begin{tabular}{lcccccc}
  \toprule
    \textbf{Model} & \textbf{CVR} & \textbf{GMV} & \textbf{Orders}\\
    \midrule
    CoMAN-B & +0.82\% &	+0.10\% &	+0.45\% \\
    CoMAN w/o AA &	+1.20\% &	+0.51\% &	+1.56\% \\
    CoMAN w/o S-t &	+2.18\% &	+3.02\% &	+3.57\% \\
    CoMAN &	+3.09\% &+3.63\% & +3.68\% \\
    \bottomrule
  \end{tabular}
\end{table}
\vspace{-2em}
\section{Conclusion}
\label{Conclusion}
In this paper, we propose a novel constrained monotonic response model CoMAN for enhancing monotonic modeling by adaptive spatio-temporal awareness in diverse takeaway marketing. Specifically, we employ two spatio-temporal perception modules to activate and capture the spatio-temporal traits in the attribute representations. Subsequently, within the monotonic layer, we design two modules that utilize spatio-temporal information to enhance adaptive learning of concavity and convexity, as well as the sensitivity function expression. Comprehensive offline and online A/B testing conducted within two marketing campaigns on the Ele.me platform demonstrates the superiority of our proposed approach. Eventually,  studies prove that our model more accurately reflects users' incentive sensitivity across different locations and periods and exhibits superior pricing performance in various spatio-temporal dimensions, thus achieving the goal of improved marketing efficiency. Furthermore, future research will delve into the finer extraction of user sensitivity differences to incentives across disparate spatio-temporal dimensions.
\bibliographystyle{ACM-Reference-Format}
\bibliography{mybibfile}

\newpage
\appendix

\section*{APPENDIX}
\section{Marketing Monotonic}

In the Ele.me's marketing campaigns, we indicate example attributes that users will be sensitive to in Figure~\ref{fig:background} below. Such monotonic characteristics are also different in different marketing tools. Specifically, in smart full discounts, the user's conversion rate will show a monotonic downward trend with respect to the monotonic increase in the full discount threshold, while with respect to the full discount amount It shows the opposite nature; in the delivery fee wavier, the user's conversion rate  shows a monotonic decreasing trend with the increase in actual paid delivery fees; in the explosive red envelope, it is similar to the intelligent full reduction, with regard to the exploding red packets. It shows a monotonically increasing trend, and the red envelope threshold shows a monotonically decreasing trend.
\begin{figure}[h]
\centering
\includegraphics[width=8.5cm]{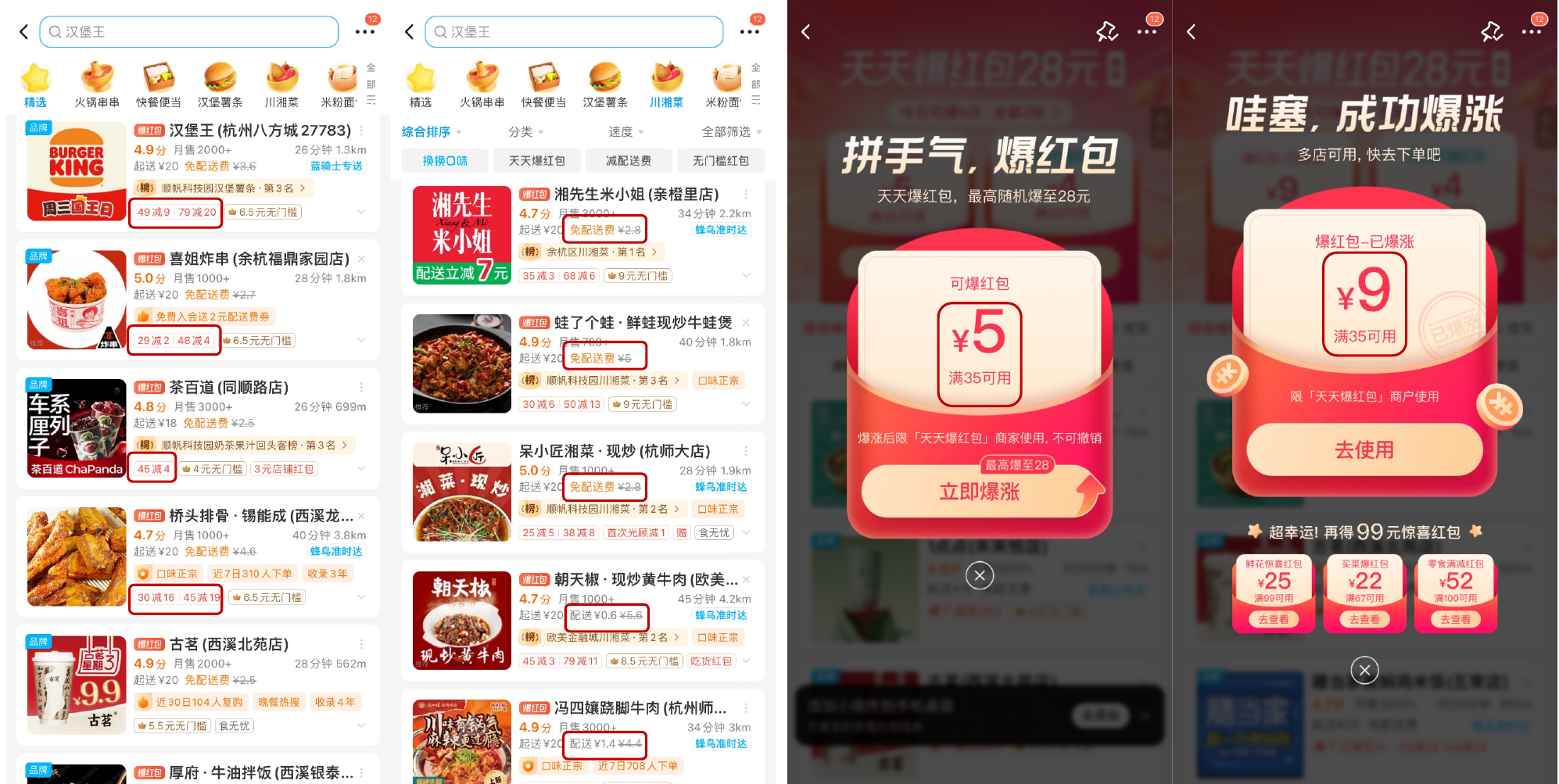}
\caption{Monotonic Properties of Marketing Campaign. Enclosed within the red box in the figure are the monotonic characteristics of various marketing campaigns. Sequentially from left to right, the attributes represented are the discount thresholds and amounts in the tiered discount, followed by the actual and nominal delivery fees in the delivery fee waiver, and concluding with the coupon amounts and thresholds in the exploding red packets.}
\label{fig:background}
\end{figure}
This paper aims to accurately characterize the sensitivity relationship between user response conversion probability and incentive attributes and to ensure that such predictions maintain high accuracy under varying spatio-temporal conditions.
% \begin{figure*}[h]
% \centering
% \includegraphics[width=14.5cm]{figs/monotonic_modeling.pdf}
% \caption{Monotonic Modeling. Road map of monotonic response modeling. Milestones in the figure: Constrain the weights\cite{archer1993application},Monotonic Networks\cite{sill1997monotonic}, Monotonic Hints\cite{sill1996monotonicity}, Monotonic Lattices\cite{milani2016fast}, DLN\cite{you2017deep}, PWL\cite{gupta2019incorporate}, COMET\cite{sivaraman2020counterexample}, MILP\cite{liu2020certified}, CMNN\cite{runje2023constrained}}
% \label{fig:monotonic_modeling}
% \end{figure*}

\section{Detailed Proofs}

%%B.1 universal approximation
\subsection{Monotone Universal Approximation}
\label{B-1}
We provide comprehensive proofs for all the theorems and lemmas presented in the main text of the paper here. The proof of the universal approximation theorem for monotone functions \cite{daniels2010monotone} is presented as follows:
\begin{theorem}
    for any continuous monotone nondecreasing function $f: K \rightarrow \mathbb{R}$, where $K$ is a compact subset of $\mathbb{R}^k$, there exists a feedforward neural network using the sigmoid as the activation function with at most $k$ hidden layers, positive weights,and output $O$ such that $|O_x - f(x)| < \epsilon$,for any $x\in K$ and $\epsilon>0$.
    \label{theorem1}
\end{theorem}

\begin{proof}
    We derive the proof by utilizing mathematical induction on the input variable $k$.  Assuming the general case where the continuous monotone undecreasing function \( f > 0 \) (otherwise, we can add a constant \( C \) and approximate \( f + C \) through the neural network output \( O \), subsequently correcting \( O \) by subtracting the constant \( C \) at the final output). If \( f \) is strictly monotone increasing and \( C^\infty \), then for \( k = 1 \), we have:
    \begin{equation}
        f(x) = \int_{0}^{\infty} \mathbf{H}(f(x) - u)du
        \label{eq16}
    \end{equation}
    where $\mathbf{H}$ is the Heavyside Function:
    \begin{equation}
        \mathbf{H}(x) = \begin{cases}
            1 & if \ x \geq 0 \\
            0 & otherwise
        \end{cases}
    \end{equation}
    Since \( f \) is continuous and increasing, it is therefore invertible. Consequently, the right side of Equation \ref{eq16} can be expressed as:
    \begin{equation}
        f(x) = \int_{0}^{\infty} \mathbf{H}(x - f^{-1}(v))dv
        \label{eq18}
    \end{equation}
    Moreover, the integral can be approximated arbitrarily closely by a Riemann sum, hence we have:
    \begin{equation}
        \sum_{i=1}^{N}(v_{i+1}-v_i)\mathbf{H}(x - f^{-1}(v))
        \label{eq19}
    \end{equation}
    Here, \([v_i]_{i=1}^N\) denotes the partition of the interval \([f(a), f(b)]\). Correspondingly, consider a fully connected neural network with \( N \) hidden neurons in a layer, where the weights of input $x$ are set to 1, the bias term are \( f^{-1}(v_i) \), and the output weights between neurons are \( v_{i+1} - v_i > 0 \). It is also noteworthy that the Heavyside function $\mathbf{H}$ can be substituted with the sigmoid activation function utilizing standard approximation parameters.

    Assume that Theorem \ref{theorem1} holds for \( k-1 \) input variables. We now combine the integral representation in Equation \ref{eq16} with the induction hypothesis. For a given \( v \), we can solve the level set equation for \( x_k \) corresponding to \( v \): $f(x_1,\ldots,x_k) = v$.
    According to the implicit function theorem, there exists a function \( g_v \) such that:
    \begin{equation}
        f(x_1, \ldots, g_v(x_1, \ldots, x_{k-1})) = v
        \label{eq20}
    \end{equation}
    where \( g_v \) is decreasing with respect to \( x_i \). This can be derived by taking the partial derivative of equation \ref{eq20} with respect to \( x_i \). Additionally, we have:
    \begin{equation}
        \mathbf{H}(f(x)-v)=\mathbf{H}(x_k - g_v(x_1, \ldots, x_{k-1}))
    \end{equation}
    Analogous to equation \ref{eq18} in the 1-D case. It remains to show that the discussion about $f(x)< v\ if\ and\ only\ if\ x_k < g_v(x_1,\ldots,x_{k-1})$ and $f(x) > v\ if\ and\ only\ if\ x_k > g_v(x_1,\ldots,x_{k-1})$.
    Similarly, by using the Riemann sum to approximate the integral in equation \ref{eq16}, we obtain the following equation analogous to equation\ref{eq19}:
    \begin{equation}
        R=\sum_{i=1}^{N}(v_{i+1}-v_i)\mathbf{H}(x_k-g_{v_i}(x_1, \ldots, x_{k-1}))
        \label{eq22}
    \end{equation}
    Due to the fact that $g_{v_i}$ is decreasing with respect to all arguments, it implies that $-g_{v_i}$ is increasing. By leveraging the inductive hypothesis, we can approximate $-g_{v_i}$ with a feedforward neural network $O_i$, wherein $x_1, \ldots, x_{k-1}$ are used as inputs. This neural network consists of $k-1$ hidden layers and employs non-negative weights, ensuring that
    \begin{equation}
        |\sum_{i=1}^{N}(v_{i+1}-v_i)\mathbf{H}(x_k-O_{i}(x_1, \ldots, x_{k-1}))-R|<\epsilon
    \end{equation}
    the sum in the above expression is infinite. The expression \ref{eq22} can be represented as a feedforward neural network with $k$ inputs and $k$ hidden layers. Here, $k-1$ hidden layers are utilized to approximate $-g_{v_i}$, while the $k$-$th$ hidden layer is employed to integrate the $N$ neural networks that produce outputs $O_i$ and accept $x_k$ as input. The weights connecting the final hidden layer to the output layer are defined as \( v_{i+1} - v_i > 0 \). Additionally, the input $x_k$ is directly skip connected to the $k$-$th$ hidden layer.
    The proof can be readily generalized to encompass continuous non-decreasing functions. For a continuous function $f$ we define its convolution with a mollifier $K_{\delta}$ as $f_\delta=f\bigotimes K_\delta$.
    
    Therefore, $f_\delta$ belongs to $C_\infty$ and converges uniformly to $f$ on compact subsets as $\delta$ approaches 0. Additionally, $f_\delta$ remains an increasing function because $K_\delta>0$. Select $\delta$ such that the inequality $|f-f_\delta|<\frac{\epsilon}{2}$ is satisfied. We then approximate $f_\delta$ using a feedforward neural network $O$ such that $|f_\delta-O|<\frac{\epsilon}{2}$. Consequently, it follows that $|f-O|<\frac{\epsilon}{2}$. If $f$ is nondecreasing, then approximate $f$ by $f_\delta$:
    \begin{equation}
        f_\delta = f + \delta(x_1+\ldots+x_k)
    \end{equation}
    The aforementioned expression is strictly monotonically increasing as $\delta$ tends to $0$.
    
\end{proof}

\begin{lemma}
    Let $\breve{\rho} \in \mathcal{\breve{A}}$. Then the Heavyside function $\mathbf{H}(x)$ can be approximated with $\tilde{\rho}_H$ on $\mathbb{R}$, where
    \begin{equation}
        \tilde{\rho}_H(x) = \alpha\tilde{\rho}(x)+ \beta
    \end{equation}
    for some $\alpha, \beta \in \mathbb{R}$ and $\alpha > 0$.
    \label{lemma1}
\end{lemma}

\begin{proof}
    Given that $\breve{\rho}$ has a lower bound, taking the limit yields:
    \begin{equation}
        c = \lim_{x \to -\infty}\breve{\rho}(x) = - \lim_{x \to +\infty}\hat{\rho}(x) < 0
    \end{equation}
    From equation \ref{m-eq11}, we have
    \begin{equation}
        \begin{aligned}
        &\lim_{x \to -\infty}\tilde{\rho}(x) = \lim_{x \to -\infty}\breve{\rho}(x) - \breve{\rho}(1) = c - \breve{\rho}(1) \\
        &\lim_{x \to +\infty}\tilde{\rho}(x) = \lim_{x \to +\infty}\hat{\rho}(x) + \hat{\rho}(1) = -(c - \breve{\rho}(1))
        \end{aligned}
    \end{equation}
    Let $\tilde{\rho}_{\mathbf{H}}(x)$ be defined as follows:
    \begin{equation}
        \tilde{\rho}_{\mathbf{H}}(x) = \frac{\tilde{\rho}(x) - c + \breve{\rho}(1)}{2(- c + \breve{\rho}(1))}
    \end{equation}
    Then
    \begin{equation}
        \lim_{a \to \infty}\tilde{\rho}_{\mathbf{H}}(a\cdot x) = \mathbf{H}(x) 
    \end{equation}
\end{proof}

\begin{lemma}
    Let $\tilde{\rho}_{\alpha,\beta}$ be an activation function for some $\alpha,\beta \in \mathbb{R}, \alpha>0$, such that for every $x\in \mathbb{R}$
    \begin{equation}
        \tilde{\rho}_{\alpha,\beta} = \alpha\tilde{\rho}(x) + \beta
    \end{equation}
    Then for every constrained montone neural network $\mathcal{N}_{\alpha, \beta}$ using $\tilde{\rho}_{\alpha, \beta}$ as an activation function $(s=(0,0,\tilde{s})$, there is a constrained monotone neural network $\mathcal{N}$ using $\tilde{\rho}$ as an activation function such that for every $\mathbf{x} \in \mathbb{R}^n$:
    \begin{equation}
        \mathcal{N}(\mathbf{x})=\mathcal{N}_{\alpha, \beta}(\mathbf{x})
    \end{equation}
    \label{lemma2}
\end{lemma}

\begin{proof}
    Let $h_k$ be the output of the $k$-$th$ constrained linear layer of a constrained monotonic neural network $\mathcal{N}_{\alpha, \beta}$ when given input $x$. For each $k>1$, we have:
    \begin{equation}
        \begin{aligned}
        \mathbf{h}_1 & =\left|\mathbf{W}_1\right|_{\mathbf{t}_1} \cdot \mathbf{x}+\mathbf{b}_1 \\
        \mathbf{h}_k & =\left|\mathbf{W}_k\right|_{\mathbf{t}_k} \cdot \mathbf{y}_{k-1}+\mathbf{b}_k \\
        \mathbf{y}_k & =\tilde{\rho}_{\alpha, \beta}\left(\mathbf{h}_k\right) \\
        \mathbf{y} & =\mathbf{h}_l
        \end{aligned}
    \end{equation}
    Let $l$ denote the total number of layers in $\mathcal{N}_{\alpha, \beta}$. Accordingly, $\mathbf{h}_k$ can be reformulated as follows:
    \begin{equation}
        \begin{aligned}
        \mathbf{h}_k &=\left|\mathbf{W}_k\right|_{\mathbf{t}_k} \cdot \mathbf{y}_{k-1}+\mathbf{b}_k \\
        & =\left|\mathbf{W}_k\right|_{\mathbf{t}_k} \cdot \tilde{\rho}_{\alpha, \beta}\left(\mathbf{h}_{k-1}\right)+\mathbf{b}_k \\
        & =\left|\mathbf{W}_k\right|_{\mathbf{t}_k} \cdot\left(\alpha \tilde{\rho}\left(\mathbf{h}_{k-1}\right)+\beta \mathbf{1}\right)+\mathbf{b}_k \quad \quad (\text {with}|\mathbf{1}|=\left|\mathbf{h}_{k-1}\right|) \\
        & =\alpha\left|\mathbf{W}_k\right|_{\mathbf{t}_k} \cdot \tilde{\rho}\left(\mathbf{h}_{k-1}\right)+\beta\left|\mathbf{W}_k\right|_{\mathbf{t}_k} \mathbf{1}+\mathbf{b}_k \\
        & =|\alpha \mathbf{W}_k|_{\mathbf{t}_k} \cdot \tilde{\rho}\left(\mathbf{h}_{k-1}\right)+\beta\left|\mathbf{W}_k\right|_{t_k} \mathbf{1}+\mathbf{b}_k \quad\  (\text {from } \alpha>0) \\
        & =\left|\mathbf{W}_k^{\prime}\right|_{t_k} \cdot \tilde{\rho}\left(\mathbf{h}_{k-1}\right)+\mathbf{b}_k^{\prime} \\
        &
        \end{aligned}
    \end{equation}  
    for $\mathbf{W}_k^\prime =\alpha \mathbf{W}_k$ and $\mathbf{b}_k\prime =\beta|\mathbf{W}_k|_{t_k}\mathbf{1}+\mathbf{b}_k$. Consequently, for each $x\in \mathbb{R}$, the output of the network $\mathcal{N}$ with weights $\mathbf{W}_1, \mathbf{W}_2^\prime, \dots, \mathbf{W}_l^\prime$ and biases $\mathbf{b}_1, \mathbf{b}_2^\prime, \dots, \mathbf{b}_l^\prime$ can be expressed as $\mathcal{N}(\mathbf{x})=\mathcal{N}_{\alpha, \beta}(\mathbf{x})$.
    
\end{proof}

\begin{theorem}
    Let $\breve{\rho} \in \mathcal{\breve{A}}$. Then any multivariate continuous monotone function f on a compact subset of $\mathbb{R}^k$ can be approximated with a monotone constrained neural network of at most $k$ layers using $\rho$ as the activation function.
\end{theorem}

\begin{proof}
    The Heaviside function can be approximated by the sigmoid function over a closed interval, as $\lim_{a \to \infty} \sigma(ax)=\mathbf{H}(x)$.
    According to Theorem \ref{theorem1} and Lemma \ref{lemma1}, any continuous monotonic function $f$ defined on a compact subset of $\mathbb{R}^k$ can be approximated with a monotonic constrained neural network of at most $k$ layers using the activation function $\tilde{\rho}_{\mathbf{H}}$ .

    From Lemma \ref{lemma2}, we can get that any continuous monotone function $f$ on a compact subset of $\mathbb{R}^k$ can be approximated with a monotone constrained neural network with at most $k$ layers using $\rho$ as the activation function.
\end{proof}

%%B.2 Convex（正文：详见appendix）
\subsection{Convex CLU}
\label{B-2}
% CLU convex
The rigorous proof of the convexity of the proposed CLU function is presented to meet the activation function assumption required for the CMNN framework.
\begin{proof}
% 引用一元函数二阶导数>0即为凸函数
The proof of concavity and convexity can be established through the second derivative test theorem as follows.

\begin{theorem}
    Let $f(x)$ be a continuously differentiable function defined on the interval $[a,b]$, possessing a second derivative on the open interval $(a,b)$.
Then for all $x\in(a,b)$, we have: 
If $f^{\prime\prime}(x)\geq0$, then the $f(x)$ is concave on the interval $[a,b]$.
If $f^{\prime\prime}(x)\leq0$, then the $f(x)$ is concave on the interval $[a,b]$.
\label{theorem_convex}
\end{theorem}

Given CLU formulated as follows:
\begin{equation}
    y=\left\{\begin{array}{ll}
-\frac{\omega_0}{2}+\frac{\omega_0}{1+e^{-\omega_1x}} &\text { if } x < 0 \\
x & \text { otherwise }
\end{array}, \omega_0>0, \omega_1>0\right.
\end{equation}

Initially, for x in the domain $x > 0$, the function simplifies to $y = x$, and its second derivative is evidently equal to $0$. Hence, our analysis will focus on the function domain where $x < 0$, and then we proceed to calculate the second derivative of CLU for $x<0$:

\begin{equation}
    \frac{\partial y}{\partial x}=\frac{\omega_0\omega_1e^{-\omega_1x}}{(1+e^{-\omega_1x})^2}
\end{equation}

\begin{equation}
\begin{aligned}
    \frac{\partial^2 y}{\partial x^2}&=\frac{(-\omega_0\omega_1^2e^{-\omega_1x})(1+e^{-\omega_1x})^2+2\omega_0(\omega_1e^{-\omega_1x})^2(1+e^{-\omega_1x})}{(1+e^{-\omega_1x})^4}\\
    &=\frac{(1+e^{-\omega_1x})(-\omega_0\omega_1^2e^{-\omega_1x}-\omega_0(\omega_1e^{-\omega_1x})^2+2\omega_0(\omega_1e^{-\omega_1x})^2)}{(1+e^{-\omega_1x})^4}\\
    &=\frac{(1+e^{-\omega_1x})(-\omega_0\omega_1^2e^{-\omega_1x}+\omega_0(\omega_1e^{-\omega_1x})^2)}{(1+e^{-\omega_1x})^4}\\
    &=\frac{(1+e^{-\omega_1x})(e^{-\omega_1x}-1)(\omega_0\omega_1^2e^{-\omega_1x})}{(1+e^{-\omega_1x})^4}\geq0
\end{aligned}
\end{equation}
\end{proof}
After verifying that the second derivative of CLU is non-negative for $x<0$, in accordance with Theorem \ref{theorem_convex}, we can conclude that CLU is a strictly monotone convex function over its entire domain.

% B.3 Methodological Property
\subsection{Methodological Property}
\begin{lemma}
    For each $i \in {1, ..., n}$ and $j \in {1, ..., m}$, We have the weight matrix for the features
    \begin{equation}
        \frac{\partial h_j}{\partial x_i} \begin{cases}
            \geq 0  &  \text{if}\ t_i = 1 \\
            \leq 0  & \text{if}\ t_i = -1
        \end{cases}
        \label{mp20}
    \end{equation}
    \label{lemma1}
\end{lemma}

\begin{proof}
    From equation \ref{mp20}, we have
    \begin{equation}
        \mathbf{h} =\left|\mathbf{W}^T\right|_{\mathbf{t}} \cdot \mathbf{x}+\mathbf{b}
    \end{equation}
    
    \begin{equation}
        h_j =\sum_i w_{i, j}^{\prime} x_i+b_j
    \end{equation}
        
    \begin{equation}
        \frac{\partial h_j}{\partial x_i} =w_{i, j}^{\prime}
    \end{equation}

    Finnaly, from equation \ref{eq10}, we have
    \begin{equation}
        \frac{\partial h_j}{\partial x_i}= \begin{cases}\left|w_{i, j}\right| \geq 0 & \text { if } t_i=1 \\ -\left|w_{i, j}\right| \leq 0 & \text { if } t_i=-1\end{cases}
    \end{equation}
    That is, the parameter $t$ can be used to indicate the monotonicity of the model.
\end{proof}

\begin{lemma}
    Let $y=\rho^s(\mathbf{h})$. Then for each $j \in {1, ..., m}$, we have $\frac{\partial y_j}{\partial h_i} \geq 0$. Moreover, if $\mathbf{s} = (m, 0, 0)$, then $\rho^\mathbf{s}_j$ is convex, and while $\mathbf{s} = (m, 0, 0)$, then $\rho^\mathbf{s}_j$ is concave.
    \label{lemma4}
\end{lemma}

\begin{proof}
    \begin{equation}
        \hat{\rho}(x)=-\breve{\rho}(-x)
    \end{equation}

    \begin{equation}
        \tilde{\rho}(x)= \begin{cases}\breve{\rho}(x+1)-\breve{\rho}(1) \quad \text { if } x<0 \\
\hat{\rho}(x-1)-\hat{\rho}(1) & \text { otherwise }\end{cases} \\
    \end{equation}

    \begin{equation}
        \rho^{\mathbf{s}}\left(h_j\right)= \begin{cases}\breve{\rho}\left(h_j\right) & \text { if } j \leq \breve{s} \\
\hat{\rho}\left(h_j\right) & \text { if } \breve{s}<j \leq \hat{s} \\
\tilde{\rho}\left(h_j\right) & \text { otherwise }\end{cases}
    \end{equation}

we have:

\begin{equation}
    \frac{\partial y_j}{\partial h_j}= \begin{cases}\breve{\rho}^{\prime}\left(h_j\right) \geq 0 & \text { if } j \leq \breve{s} \\ \breve{\rho}^{\prime}\left(-h_j\right) \geq 0 & \text { if } \breve{s}<j \leq \hat{s} \\ \breve{\rho}^{\prime}\left(h_j+1\right) \geq 0 & \text { if } \breve{s}+\hat{s}<j \text { and } h_j<0 \\ \breve{\rho}^{\prime}\left(1-h_j\right) \geq 0 & \text { if } \breve{s}+\hat{s}<j \text { and } h_j \geq 0\end{cases}
\end{equation}

if $\mathbf{s}=(m,0,0)$, we have the convex function as follows:
\begin{equation}
    \rho^s(\mathbf{h}) = \breve{\rho}(h_j)
\end{equation}

Similarly,if $\mathbf{s}=(0,m,0)$, we have the concave function as follows:
\begin{equation}
    \rho^s(\mathbf{h}) = \hat{\rho}(h_j)
\end{equation}

\end{proof}

\begin{corollary}
    From Lemma \ref{lemma1} and Lemma \ref{lemma4}, for each $i \in {1, ..., n}$ and $j \in {1, ..., m}$, we can get that if $t_i = 1$, then $\frac{\partial y_i}{\partial x_i} \geq 0$ and while $t_i=-1$, then $\frac{\partial y_i}{\partial x_i} \leq 0$. Moreover, if $\mathbf{s} = (m, 0, 0)$, then $\rho^\mathbf{s}_j$ is convex, else if $\mathbf{s} = (m, 0, 0)$, then $\rho^\mathbf{s}_j$ is concave.
\end{corollary}

\section{Implemental Details}
\label{appendix-c}
In Table \ref{tab:imp_details}, we present the optimal architecture of our model as derived from experimental results. In the embedding layer, we distinguish between the discretization of categorical ID features and numerical features, setting the embedding dimensions for these features to 8 and 4, respectively. Additionally, in the S-t attention module and the FPM, we carefully select the final activation functions for each parameter. Specifically, for the MLP networks learning \(\omega_1\) and \(\omega_2\), we employ a combination of Leaky ReLU and Softplus to learn non-negative parameters, whereas for \(\omega_0\) and \(\omega_3\), we utilize the sigmoid function to learn the probability.
\begin{table}[!htbp]
  \caption{Training configurations of the model architecture.}
  \label{tab:imp_details}
  \centering
  \begin{tabular}{lcp{2.1cm}cccc}
  \toprule
    \textbf{Module}& \textbf{Size} & \textbf{Activation}& \textbf{Dropout}\\
    \midrule
    Embedding &	$400 \times \{8,4\}$ & - &	- \\
    Experts Net &	$512 \times 256 \times 128$ & ReLU & 0.3 \\
    Gating Net &	$128 \times 64 \times 3$ & \{ReLU, Softmax\} & 0.3 \\
    Tower Net &	$128 \times 64$ & ReLU & 0.3 \\
    CMNN &	$64 \times 1$ & \{CLU, FPM\} & 0 \\
    S-t &	$32 \times 1$ & \{LeakyReLU, softplus($\omega_1$, $\omega_2$), sigmoid($\omega_0$, $\omega_3$)\} & 0.3 \\
    AA &	$64 \times 3$ & ReLU & 0.3 \\
    \bottomrule
  \end{tabular}
\end{table}

% \section{Additional Results}
% more experiments results









\end{document}